\documentclass[preprint,12pt]{elsarticle}

\usepackage[utf8]{inputenc}
\usepackage{amsthm}
\usepackage{booktabs}
\usepackage{float}
\usepackage[utf8]{inputenc}
\usepackage{dirtytalk}
\usepackage[dvipsnames]{xcolor}
\usepackage{algorithm}
\usepackage{svg}
\usepackage[noend]{algpseudocode}
\usepackage{subfig}
\usepackage{enumitem}
\usepackage{tikz}
\usepackage{multirow}
\usepackage[export]{adjustbox}
\usepackage{amsmath}
\usepackage{amssymb}
\usepackage{float}
\usepackage{svg}
\usepackage{enumitem}
\usepackage{makecell}
\usepackage{graphicx}

\theoremstyle{definition}
\newtheorem{definition}{Definition}[section]
\newtheorem{lemma}{Lemma}[section]
\newtheorem{theorem}{Theorem}[section]
\newtheorem{corollary}{Corollary}[section]
\newtheorem{proposition}{Proposition}[section]
\newtheorem{example}{Example}[section]

\journal{arxiv.org}

\begin{document}

\begin{frontmatter}



\title{Multi-class granular approximation by means of disjoint and adjacent fuzzy granules}


\author{ Marko Palangeti\' c$^{a}$, Chris Cornelis$^a$,  Salvatore Greco$^{b,c}$, Roman S\l{}owi\' nski$^{d,e}$\\}

\address{$^a$Department of Applied Mathematics, Computer Science and Statistics, \\ Ghent University, Ghent, Belgium, \{marko.palangetic, chris.cornelis\}@ugent.be  \\
$^b$Department of Economics and Business, University of Catania, Catania, Italy,\\  salgreco@unict.it\\
$^c$Portsmouth Business School, Centre of Operations Research and Logistics (CORL), \\
University of Portsmouth, Portsmouth, United Kingdom \\
$^d$Institute of Computing Science,  Pozna\'n University of Technology, Pozna\'n, Poland,\\ roman.slowinski@cs.put.poznan.pl \\
$^e$Systems Research Institute, Polish Academy of Sciences, Warsaw, Poland \\}

\address{}

\begin{abstract}
In granular computing, fuzzy sets can be approximated by granularly representable sets that are as close as possible to the original fuzzy set w.r.t. a given closeness measure. Such sets are called granular approximations. In this article, we introduce the concepts of disjoint and adjacent granules and we examine how the new definitions affect the granular approximations. First, we show that the new concepts are important for binary classification problems since they help to keep decision regions separated (disjoint granules) and at the same time to cover as much as possible of the attribute space (adjacent granules). 
Later, we consider granular approximations for multi-class classification problems leading to the definition of a multi-class granular approximation. Finally, we show how to efficiently calculate multi-class granular approximations for Łukasiewicz fuzzy connectives. We also provide graphical illustrations for a better understanding of the introduced concepts. 
\end{abstract}


\end{frontmatter}


\section{Introduction}

Granular computing is a paradigm in information processing which includes a segmentation of complex information into smaller pieces called \textit{information granules} \cite{yao2013granular,pedrycz2014allocation,bargiela2006roots}. An information granule (or just a \textit{granule}) is a collection of instances that can be interpreted jointly. For example, an image of a human body can be disentangled into certain body parts that have precise meanings. Also, those parts can be later segmented into even smaller meaningful parts, etc. The previous example also shows the hierarchical nature of granulation, i.e., the definition of granules depends on the level of detail that we want to capture. Granules are usually constructed based on a common association (indiscernibility, similarity, functionality, proximity, coherency etc.) of instances \cite{degang2011granular,yao1999granular}. 

Fuzzy logic and fuzzy set theory are used to model partial truth of logical expressions \cite{zadeh1965}. In other words, the expression is not only true or false, but it possesses a degree of truth represented by a value from interval $[0,1]$. Value $0$ stands for a completely false statement, while value 1 stands for a completely true statement. With the help of fuzzy logic, one can introduce the concept of a fuzzy granule \cite{zadeh1979fuzzy} where every instance has a degree of membership to a certain granule. Fuzzy granules are useful when it is hard to determine sharp boundaries of pieces obtained from a disentanglement of a complex object. In such case, soft boundaries are expressed using fuzzy sets.

Lotfi Zadeh identified granulation as one of three basic concepts in underlying human cognition \cite{zadeh1997toward}, the other two being organization and causation. While organization represents the integration of parts into a whole, granulation refers to the opposite process. With fuzziness as a key part of the granulation in human cognition, humans are able to make reasonable decisions in a world that is characterized with partial knowledge, partial certainty, partial truth and imprecision in general.

In this article, granules are identified in \textit{information tables} based on the concept of data consistency, following our previous work in \cite{palangetic2021granular,palangetic2021novel}. Assume we have a prediction problem where we want to assign a decision label to a given instance described by condition attributes. In this setting, we say that two instances are consistent w.r.t. a given relation, if their relation on the condition attributes implies the same type of relation on the decision attribute. The relations that we consider here are (fuzzy) indiscernibility and (fuzzy) dominance. Based on that, an instance is consistent in a dataset if it is consistent with all other instances. For a consistent instance (w.r.t. a given relation), a granule is formed as a conjunction of two concepts:
\begin{itemize}
    \item the set of instances that relate to the given consistent instance, and
    \item the association of the consistent instance to a particular decision.
\end{itemize}
Due to consistency, the instances that relate (w.r.t. a given relation like e.g. indiscernibility or dominance) to a given consistent instance will be associated to the same decision or to a decision that relates to the decision of the consistent instance. In a classification problem, we have decision classes and association of the consistent instance refers to the membership of the instance to a decision class. In a regression problem, the association refers to the numerical value that the consistent instance takes in the decision attribute. 

In practice, due to perturbation in data caused by incomplete knowledge or by random effects that occur during data generation, datasets contain instances that are not consistent. To make a dataset consistent, different approaches have been applied. The best-known symbolic approach to this issue is the rough set approach (or the indiscernibility-based rough set approach (IRSA)) and its generalizations like the dominance-based rough set approach (DRSA) and fuzzy rough sets (fuzzy IRSA and fuzzy DRSA) \cite{pawlak1982rough,greco1998new,dubois1990rough,palangetic2021granular}. The disentanglement of crisp rough sets into granules was discussed in \cite{yao1999granular} and of fuzzy rough sets in \cite{degang2011granular,wang2015granular,fang2019granular}. The IRSA and DRSA were integrated into the preorder-based rough set approach (PRSA), and their fuzzy counterparts into the fuzzy PRSA. A more comprehensive overview is provided in \cite{palangetic2021granular}.

On the other hand, removing inconsistencies using machine learning w.r.t. a crisp preorder relation is covered in \cite{kotlowski2008statistical} and w.r.t. a fuzzy preorder relation in \cite{palangetic2021novel}. The result of the former approach is called \textit{monotone approximation} (due to the monotonicity properties of the granules), while the result of the latter approach is called \textit{granular approximation}. These machine learning approaches include an optimization procedure that removes inconsistencies at the least possible cost (w.r.t. a certain loss function).

In this article, we extend the definition of granular approximations to the multi-class context. We introduce concepts of disjoint and adjacent fuzzy granules and we discuss how these concepts relate to the formerly introduced granular approximations. They are important in classification problems since they help us to keep decision regions separated (disjoint granules) while covering as much as possible of an attribute space (adjacent granules). Then, we formulate an optimization procedure in order to extend granular approximations to the multi-class classification problem leading to the definition of \textit{multi-class granular approximations}.
Such approximation is a union of granules constructed in the way described above; it is a fuzzy set constructed as a conjunction of a fuzzy relation and an association value. These association values, as discussed in \cite{palangetic2021novel}, can be interpreted as the degree up to which an instance belongs to a certain decision class.

In some classification problems, application of ``the degree up to which an instance belongs to a class" may not be possible, since a decision class may be defined in a strictly discrete way, i.e., an instance belongs to a decision class or not. Examples of the case where consideration of membership degrees is more appropriate are recommender systems. A degree of preference for a certain product by a user can be modeled using values between 0 and 1. It is also possible that collected data contain only binary preferences (e.g., likes and dislikes) while the underlying degree of preference is hidden and can be estimated using machine learning techniques.

The remainder of this paper is structured as follows. In Section \ref{sec:preliminaries}, we recall some useful preliminaries about fuzzy logic theory, fuzzy rough and granularly representable sets, and granular approximations. Section \ref{sec:disjoint_and_adjacent_granules} deals with the concept of disjoint granules and adjacent granules. It provides definitions of the concepts together with an analysis of how these definitions pertain to the granular approximations introduced in the previous section. Section \ref{sec:case_of_a_classification_problem} explains how the concepts from the previous sections can be applied in binary and multi-class classification problems, and it introduces the definition of a multi-class granular approximation. Section \ref{sec:calculation} presents how to efficiently calculate multi-class granular approximations and provides a graphical illustration of how the granules look in practice. Section \ref{sec:conclusion_and_future_work} contains the conclusion and outlines future work.

\section{Preliminaries}
\label{sec:preliminaries}
\subsection{Fuzzy logic connectives}
In this subsection, the definitions and terminology are based on \cite{klement2013triangular}. Recall that {\em $t$-norm} $T:[0,1]^2 \rightarrow [0,1]$ is a binary operator which is commutative, associative, non-decreasing in both arguments, and $ \forall x \in [0,1] ,\, T(x,1) = x$. Since a $t$-norm is associative, we may extend it unambiguously to a $[0,1]^n \rightarrow [0,1]$ mapping for any $n > 2$. Some commonly used $t$-norms are listed in the left-hand side of Table \ref{table:tnorms}.

\begin{table}[H]
\begin{adjustbox}{width=\columnwidth,center}
    \begin{tabular}{c|rcl|rcl}
    Name & Definition & & & R-implicator\\
    \hline
Minimum &         $T_M(x,y)$ &=& $\min(x,y)$  & $I_{T_M}(x,y)$ &=& $\left\{\begin{array}{cc}
       1 & \mbox{if $x \le y$}  \\
            y & \mbox{otherwise} 
        \end{array}
        \right.$    \\
Product        & $T_P(x,y)$ &=& $xy$ & $I_{T_P}(x,y)$ &=& $\left\{\begin{array}{cc}
       1 & \mbox{if $x \le y$}  \\
            \frac{y}{x} & \mbox{otherwise} 
        \end{array}
        \right.$ \\
{\L}ukasiewicz    &    $T_L(x,y)$ &=& $\max(0,x+y-1)$ &    $I_{T_L}(x,y)$ &=& $\min(1,1-x+y)$ \\
Drastic     &   $T_D(x,y)$ &=& $\left\{\begin{array}{cc}
               \min(x,y) & \mbox{if $\max(x,y) = 1$}  \\
            0 & \mbox{otherwise} 
        \end{array}
        \right.$  &   $I_{T_D}(x,y)$ &=& $\left\{\begin{array}{cc}
           y & \mbox{if $x = 1$}  \\
            1 & \mbox{otherwise} 
        \end{array}
        \right.$\\
\makecell{Nilpotent \\ minimum} & $T_{nM}(x,y)$ &=& $\left\{\begin{array}{cc}
            \min(x,y) & \mbox{if $x + y > 1$}  \\
            0 & \mbox{otherwise} 
        \end{array}
        \right.$ & $I_{T_{nM}}(x,y)$ &=& $\left\{\begin{array}{cc}
            1 & \mbox{if $x \le y$}  \\
            \max(1-x,y) & \mbox{otherwise} 
        \end{array}
        \right.$\\
        \hline
    \end{tabular}
    \end{adjustbox}
    \vspace{8pt}
    \caption{Some common $t$-norms and their R-implicators}
    \label{table:tnorms}
\end{table}

An {\em implicator}  (or {\em fuzzy implication}) $I : [0,1]^2 \rightarrow [0,1]$ is a binary operator which is non-increasing in the first component, non-decreasing in the second one and such that $I(1,0) = 0$ and $I(0,0) = I(0,1) = I(1,1) = 1$. 
The residuation property holds for a $t$-norm $T$ and implicator $I$ if for all $x, y, z \in [0,1]$, it holds that
$$
T(x,y) \leq z \Leftrightarrow x \leq I(y,z).
$$ 
It is well-known that the residuation property holds if and only if $T$ is left-continuous and $I$ is defined as the residual implicator (R-implicator) of $T$, that is
$$I_T(x,y) = \sup \{\beta  \in [0,1] : T(x,\beta) \leq y\}.$$

The right-hand side of Table \ref{table:tnorms} shows the residual implicators of the corresponding $t$-norms. Note that all of them, except $I_{T_D}$, satisfy the residuation property.

If $T$ is a continuous $t$-norm, it is \textit{divisible}, i.e., for all $x, y \in [0,1]$, it holds that
$$
\min(x, y) = T(x, I(x, y)) = T(y, I(y, x)).
$$

A {\em negator}  (or { \em fuzzy negation}) $N : [0,1] \rightarrow [0,1]$ is a unary and non-increasing operator for which it holds that $N(0) = 1$ and $N(1) = 0$. A negator is involutive if $N(N(x)) = x$ for all $x \in [0,1]$. The standard negator, defined as 
$$
N_s(x) = 1-x,
$$ 
is involutive.

For 
 implicator $I$, we define the negator induced by $I$ as $N(x) = I(x, 0)$. We will call triplet $(T, I, N)$, obtained as previously explained, a residual triplet. For a residual triplet we have that the following properties hold for all $x, y, z \in [0,1]$:
\begin{subequations}
\begin{align}
&\bullet    &&T(x, y) \leq x \quad \text{and} \quad T(x,y) \leq y, &&\label{eq:t-norm_smaller_parameters} \\
&\bullet    &&I(x, y) \geq y, 
&&\label{eq:implicator_greater_second_parameter} \\
&\bullet    &&T(x, I(x, y)) \leq y, 
&&\label{eq:modus_ponens} \\
&\bullet    &&x \leq y \Leftrightarrow I(x,y) = 1, \,  \text{(ordering property)} &&\label{eq:ordering_property} \\
&\bullet    &&T(x, I(y, z)) \leq I(I(x,y), z), &&\label{eq:t_norm_implicator_property2}\\
&\bullet    &&I(T(x,y), z) = I(x, I(y, z)), &&\label{eq:t_norm_implicator_property} \\
&\bullet    &&T(x, N(y)) \leq N(I(x,y)) \, \text{(consequence of (\ref{eq:t_norm_implicator_property2}) when $z=0$),} &&\label{eq:implicator_negator_property}\\
&\bullet    &&N(T(x,y)) = I(x, N(y)) \, \text{(consequence of (\ref{eq:t_norm_implicator_property}) when $z=0$)}.
&&\label{eq:implicator_negator_property2}
\end{align}
\end{subequations}
Two fuzzy binary operators $B^1$ and $B^2$ are isomorphic if there exists a bijection $\varphi:[0,1] \rightarrow [0,1]$ such that $B^1 = \varphi^{-1}(B^2(\varphi(x), \varphi(y)))$ while unary operators $V^1$ and $V^2$ are isomorphic if $V^1 = \varphi^{-1}(V^2(\varphi(x)))$. Moreover, we write $B^1 \equiv B^2_{\varphi}$ and $V^1 \equiv V^2_{\varphi}$.

If residual triplet $(T, I, N)$ is generated by $t$-norm $T$, then the residual triplet generated by $T_{\varphi}$ is $(T_{\varphi}, I_{\varphi}, N_{\varphi})$.

A $t$-norm for which the induced negator of its R-implicator is involutive is called an IMTL $t$-norm. In Table \ref{table:tnorms}, $T_L$ and $T_{nM}$ are IMTL $t$-norms where the corresponding induced negator is $N_s$. A residual triplet $(T, I, N)$ that is generated with an IMTL $t$-norm is called an IMTL triplet. If $(T, I, N)$ is an IMTL triplet, then $(T_{\varphi}, I_{\varphi}, N_{\varphi})$ is also an IMTL triplet. 

For an IMTL triplet, we have that the following property holds for all $x, y, z \in [0,1]$:
\begin{align}
&&I(N(x), N(y)) = I(y, x) &&\label{eq:implicator_negator_property3}
\end{align}

A continuous $t$-norm is IMTL if and only if it is isomorphic to the Łukasiewicz $t$-norm. Such $t$-norm is \textit{strongly max-definable}, i.e., for all $x,y \in [0,1]$, it holds that
\begin{equation}
\label{eq:strong_max_definability}
\max(x, y) = I(I(x, y), y) = I(I(y, x), x).
\end{equation}
A residual triplet generated by a $t$-norm isomorphic to $T_L$, $T_{L,\varphi}$ is denoted by $(T_{L,\varphi}, I_{L, \varphi}, N_{L, \varphi})$. Note that $N_L \equiv N_s$.

\subsection{Fuzzy sets and fuzzy relations}

Given a non-empty universe set $U$, a fuzzy set $A$ on $U$ is an ordered pair $(U, m_A)$, where $m_A:U \rightarrow [0,1]$ is a membership function that indicates how much an element from $U$ is contained in $A$. Instead of $m_A(u)$, the membership degree is often written as $A(u)$. If the image of $m_A$ is $\{0,1\}$ then $A$ is a crisp (ordinary) set. For negator $N$, the fuzzy complement $coA$ is defined as $coA(u) = N(A(u))$ for $u\in U$. If $A$ is crisp then $coA$ reduces to the standard complement. For $\alpha \in (0,1]$, the $\alpha$-level set of fuzzy set $A$ is a crisp set defined as $A_{\alpha} = \{u \in U; A(u) \geq \alpha \}$.

A fuzzy relation $\widetilde{R}$ on $U$ is a fuzzy set on $U \times U$, i.e., a mapping $\widetilde{R}: U \times U \rightarrow [0,1]$ which indicates how much two elements from $U$ are related. Some relevant properties of fuzzy relations include:
\begin{itemize}
    \item $\widetilde{R}$ is reflexive if $\forall u \in U, \widetilde{R}(u,u) = 1$.
    \item $\widetilde{R}$ is symmetric if $\forall u,v \in U,\ \widetilde{R}(u,v) = \widetilde{R}(v,u)$.
    \item $\widetilde{R}$ is $T$-transitive w.r.t. $t$-norm $T$ if $\forall u,v,w \in U$ it holds that \\
    $T(\widetilde{R}(u,v),\widetilde{R}(v,w)) \leq \widetilde{R}(u,w)$.
\end{itemize}{}
A reflexive and $T$-transitive fuzzy relation is called a $T$-preorder relation while a symmetric $T$-preorder is a $T$-equivalence relation.

To illustrate some of these fuzzy relations, we assume that instances from $U$ are described with a finite set of numerical attributes $Q$. For attribute $q \in Q$, Let $u^{(q)}$ and $v^{(q)}$ be the evaluations of instances $u$ and $v$ on attribute $q$.
An example of a $T_L$-preorder relation (expressing dominance) on attribute $q$, given in \cite{palangetic2021fuzzy}, is
\begin{equation}
    \label{eq:triangular_dominance}
    \widetilde{R}_q^\gamma(u,v) = \max \left (\min \left (1 -\gamma \frac{v^{(q)} - u^{(q)}}{range(q)},1 \right ), 0 \right),
\end{equation}{}
where $\gamma$ is a positive parameter and $range(q)$ is the difference between the maximal and minimal value on $q$. The illustration of such relation is given in Figure \ref{fig:preorder_illustration}. On the figure, we have attribute $q$ with range equal to 1. The value of $v^{(q)}$ is $0.5$, while the value of $u^{(q)}$ goes from 0 to 1. For every $u^{(q)}$ and for $\gamma = 3$, the value of $\widetilde{R}_q^\gamma(u,v)$ is calculated and depicted.
\begin{figure}[H]
    \centering
    \includegraphics[width = .5\textwidth]{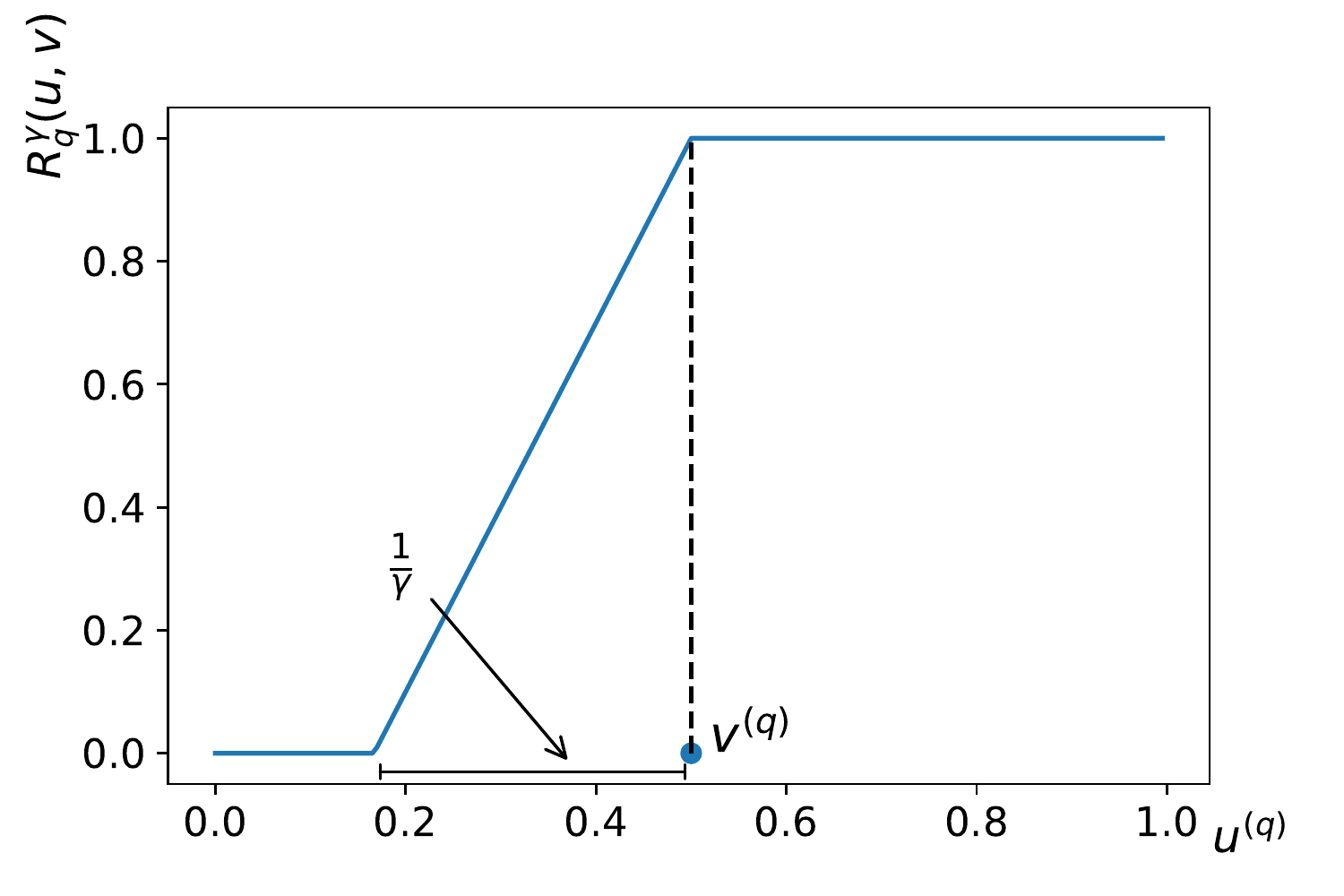}
    \caption{Illustration of the $T$-preorder relation on criterion $q$ for pair of objects $(u,v)$}
    \label{fig:preorder_illustration}
\end{figure}
An example of a $T_L$-equivalence relation (expressing indiscernibility) on the same attribute is
\begin{align}
\label{eq: triangular similarity}
   \widetilde{R}_q^{\gamma} (u,v) = \max \left ( 1 - \gamma \frac{|u^{(q)} - v^{(q)}|}{range(q)},0\right),
\end{align}
while the illustration of such relation is given in Figure \ref{fig:equivalence_illustration}. The description of the figure is same as for Figure \ref{fig:preorder_illustration}.
\begin{figure}[H]
    \centering
    \includegraphics[width = .5\textwidth]{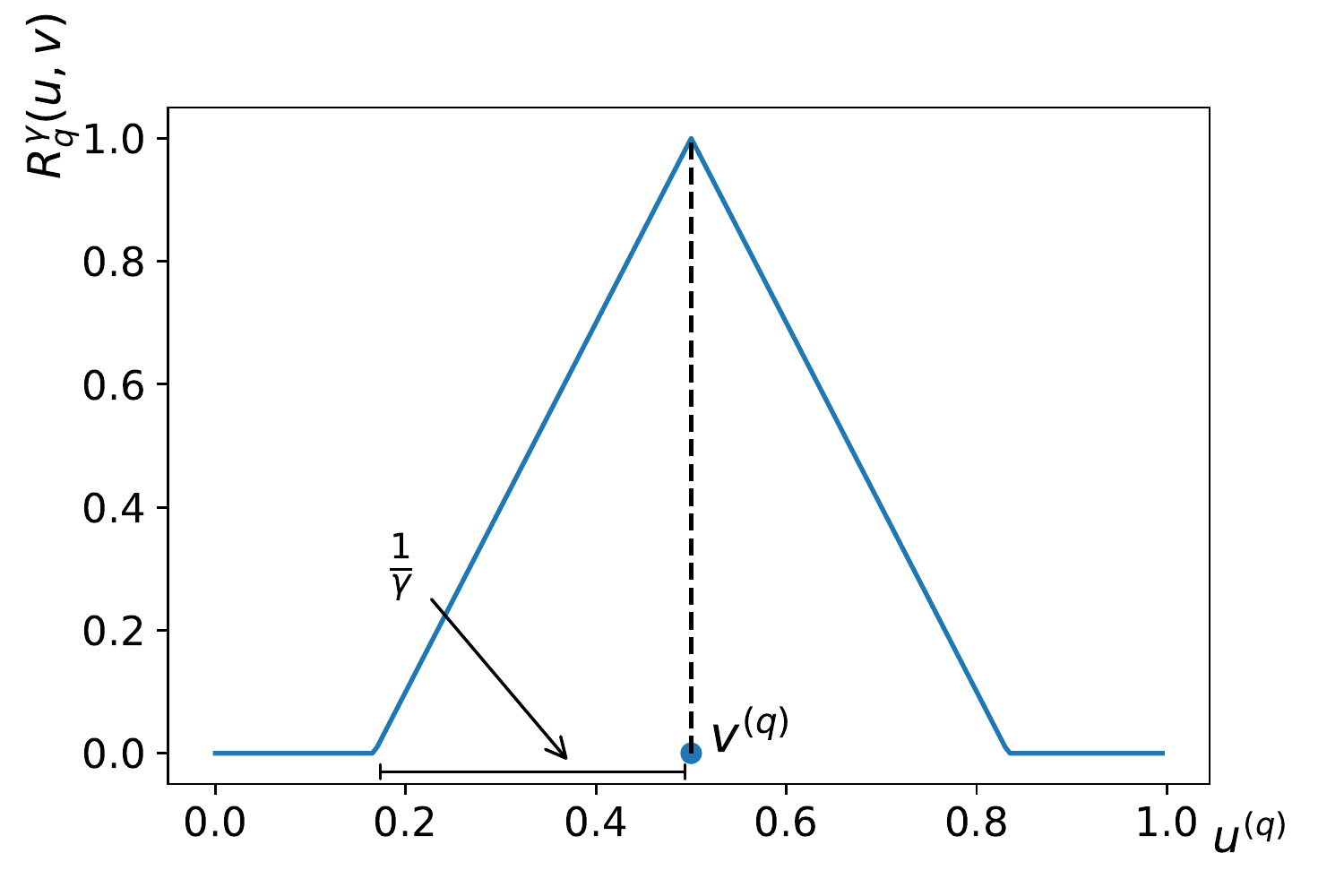}
    \caption{Illustration of the $T$-equivalence relation on criterion $q$ for pair of objects $(u,v)$}
    \label{fig:equivalence_illustration}
\end{figure}
In both cases, the relation over all attributes from $Q$ is defined as $\widetilde{R}(u,v) = \min_{q \in Q} \widetilde{R}_q(u,v)$.

\subsection{Fuzzy rough and granularly representable sets}
This subsection is based on \cite{palangetic2021granular}.
Let $U$ be a finite set of instances, $A$ a fuzzy set on $U$ and $\widetilde{R}$ a $T$-preorder relation on $U$. The fuzzy PRSA lower and upper approximations of $A$ are fuzzy sets for which the membership function is defined as:
\begin{equation}
\label{eq:fuzzy rough approx}
\begin{aligned}
\underline{\text{apr}}_{R}^{\min, I}(A)(u) = \min \{I(\widetilde{R}(v,u), A(v)); v \in U\} \\
\overline{\text{apr}}_{R}^{\max, T}(A)(u) = \max \{T(\widetilde{R}(u,v), A(v)); v \in U\}.
\end{aligned}
\end{equation}
For an IMTL $t$-norm and the corresponding implicator $I$ and negator $N$, we have the well known duality property.
\begin{align}
\label{eq:duality}
    \begin{split}
    N (\underline{\text{apr}}_{R}^{\min, I}(A)(u) ) &=  \overline{\text{apr}}_{R}^{\max, T}(coA)(u) \\ 
    N (\overline{\text{apr}}_{R}^{\max, T}(A) (u)) &=  
    \underline{\text{apr}}_{R}^{\min, I}(coA) (u)
    \end{split}
\end{align}

A fuzzy granule with respect to fuzzy relation $\widetilde{R}$ and parameter $\lambda \in [0,1]$ is defined as a parametric fuzzy set
\begin{align}
\label{eq:granule}
    \widetilde{R}^+_\lambda (u) = \{(v, T(\widetilde{R}(v,u), \lambda ); v \in U\}.
\end{align}
while the granule with respect to the inverse fuzzy relation $\widetilde{R}^{-1}$ is 
\begin{align}
\label{eq:granule2}
    \widetilde{R}^-_\lambda (u) = \{(v, T(\widetilde{R}(u,v), \lambda ); v \in U\}.
\end{align}
Here, parameter $\lambda$ describes the association of instance $u$ to a particular decision. For example, in classification problems, it represents the membership degree of $u$ to  a particular decision class. 
A fuzzy set $A$ is granularly representable (GR) w.r.t. relation $\widetilde{R}$ if 
$$
A = \bigcup \{ \widetilde{R}^+_{A(u)}(u); u \in U\},
$$
where the union is defined using $\max$ operator.
Some equivalent forms to define granular representability are such that for all $u,v \in U$:
\begin{equation}
\label{eq:granular_rep}
T(\widetilde{R}(v,u), A(u)) \leq A(v) \Leftrightarrow \widetilde{R}(v,u) \leq I(A(u), A(v)) 
\end{equation}

\begin{proposition} \cite{palangetic2021novel}
\label{prop:complement_representable}
If fuzzy set $A$ is granularly representable w.r.t. $T$-preorder relation $\Tilde{R}$, then $coA$ is granularly representable w.r.t.  $\Tilde{R}^{-1}$.
\end{proposition}
\begin{proposition}
\label{prop:l_and_s_gp} \cite{palangetic2021granular}
It holds that $\underline{\text{apr}}_{\Tilde{R}}^{\min, I}(A)$ is the largest GR set contained in $A$, while $\overline{\text{apr}}_{\Tilde{R}}^{\max, T}(A)$ is the smallest GR set containing $A$.
\end{proposition}

\subsection{Granular approximations}
This subsection is based on \cite{palangetic2021novel}.
A granular approximation is a granularly representable set that is as close as possible to the observed fuzzy set (set that is approximated) with respect to the given closeness criterion. The closeness is measured by a loss function $L:\mathbb{R} \times \mathbb{R} \rightarrow \mathbb{R}^+$. For a given loss function $L$, fuzzy set $A$, relation $\widetilde{R}$ and residual triplet $(T, I, N)$, the granular approximation $\hat{A}$ is obtained as a result of the following optimization problem:
\begin{equation}
\label{eq:general_optimization}
\begin{aligned}
&\text{minimize}  && \displaystyle\sum_{u \in U}  L(A(u), \hat{A}(u)) \\
&\text{subject to}    && T( \widetilde{R}(u,v), \hat{A}(v)) \leq \hat{A}(u), \quad   u, v \in U\\
  &              &&0 \leq \hat{A}(u) \leq 1, \quad u \in U.
\end{aligned}
\end{equation}
The objective function in (\ref{eq:general_optimization}) ensures that the resulting fuzzy set $\hat{A}$ is as close as possible to the given fuzzy set $A$ (w.r.t. loss function $L$) while the constraints of (\ref{eq:general_optimization}) guarantee that $\hat{A}$ is granularly representable. We recall two well-known loss functions; $p$-quantile loss:
\begin{equation}
    L_{p} = (y, \hat{y}) = (y - \hat{y}) (p - \mathbf{1}_{y-\hat{y} < 0}) = \begin{cases}
        p|y - \hat{y}| & \text{if   }  y-\hat{y} > 0, \\
        (1-p)|y - \hat{y}| & \text{otherwise},
        \end{cases} \label{eq:quantile_loss}
\end{equation}
and mean squared error:
\begin{equation}
L_{MSE}(y, \hat{y}) = (y - \hat{y})^2 \label{eq:mse}.
\end{equation}
The $p$-quantile loss for $p=\frac{1}{2}$ is called mean absolute error:
\begin{equation}
    L_{MSE}(y, \hat{y}) = |y - \hat{y}| \label{eq:mae}
\end{equation}

It was shown that optimization problem (\ref{eq:general_optimization}) can be efficiently solved if $T$ is isomorphic to $T_L$ ($T = T_{L,\varphi}$) and for $L$ being the scaled $p$-quantile loss: $L_{p,\varphi} = L_p(\varphi(y), \varphi(\hat{y}))$ or the scaled mean squared error: $L_{MSE,\varphi} = L_{MSE}(\varphi(y), \varphi(\hat{y}))$.

\begin{definition}
Loss function $L$ is \textit{symmetric} if $L(y, \hat{y}) = L(\hat{y}, y)$.
\end{definition}
It is easy to verify that $L_{MSE,\varphi}$ and $L_{MAE, \varphi}$ are symmetric loss functions, while $L_{p, \varphi}$ for $p \neq \frac{1}{2}$ is not. However, it can be observed that $L_{p, \varphi}(y, \hat{y}) = L_{1 - p, \varphi}(\hat{y}, y)$.

\begin{definition}
\label{def:v-type}
We say that loss function $L$ is \textit{of $\lor$-type} if for any real number $a$, it holds that
\begin{itemize}
    \item $L(a,a) = 0$,
    \item functions $L(x,a)$ and $L(a,x)$ are increasing for $x > a$ and
    \item functions $L(x,a)$ and $L(a,x)$ are decreasing for $x < a$.
\end{itemize}
\end{definition}
The previous definition says that the loss is greater if $x$ is more distant from $a$. It is easy to verify that the mean squared error and $p$-quantile loss for $p \in (0,1)$ are of $\lor$-type. The $p$-quantile loss for $p \in \{0, 1\}$ is not of $\lor$-type since $L_0(a,x) = 0$ for $x < a$ and $L_1(a, x) = 0$ for $x > a$. 

\begin{definition}
\label{def:perserving_duality}
A loss function $L:[0,1] \times [0,1] \rightarrow \mathbb{R}^+$ is \textit{$N$-duality preserving} if $L(y, \hat{y}) = L(N(\hat{y}), N(y))$ for $N$ from the residual triplet $(T, I, N)$.
\end{definition}
In \cite{palangetic2021novel}, it was shown that both $L_{p, \varphi}$ and $L_{MSE,\varphi}$ are $N$-duality preserving for IMTL triplet $(T_{L, \varphi}, I_{L, \varphi}, N_{L, \varphi})$.

\section{Disjoint and adjacent granules}
\label{sec:disjoint_and_adjacent_granules}
\subsection{Definitions and basic properties}
Note that from now on we assume that $\widetilde{R}$ is a $T$-preorder relation for a residual triplet $(T, I, N)$. 

\begin{definition}
Two fuzzy sets $A$ and $B$, defined on universe $U$, are called $T$-disjoint if 
$$
T(A(u), B(u)) = 0 \text{ for every } u \in U.
$$
\end{definition}
For the fuzzy granules, we have the following property:
\begin{proposition}
\label{prop: t disjoint}
Let $u,v \in U$. Two fuzzy granules $\widetilde{R}^+_{\lambda_1}(u)$ and $\widetilde{R}^-_{\lambda_2}(v)$ are $T$-disjoint if and only if
\begin{equation}
T(\lambda_1, \lambda_2) \leq N(\widetilde{R}(v,u)).   \label{eq:T-disjoint_condition-initial} 
\end{equation}
\end{proposition}
\begin{proof}
The statement that two granules are $T$-disjoint is equivalent to:
\begin{align*}
&\max_{w \in U} T(T(\widetilde{R}(w, u), \lambda_1), T(\widetilde{R}(v, w), \lambda_2)) =  0 \\
\Leftrightarrow & \max_{w \in U} T(T(\widetilde{R}(v, w), \widetilde{R}(w, u)) , T(\lambda_1, \lambda_2)) = 0 \\
\Leftrightarrow & T\left( \max_{w \in U} T(\widetilde{R}(v, w), \widetilde{R}(w, u)), T(\lambda_1, \lambda_2)\right) = 0 \\
\Leftrightarrow & T( \widetilde{R}(u, v), T(\lambda_1, \lambda_2)) = 0 \\
\Leftrightarrow & T(\lambda_1, \lambda_2) \le I(\widetilde{R}(v, u),0) \\
\Leftrightarrow & T(\lambda_1, \lambda_2) \leq N(\widetilde{R}(v,u)).
\end{align*}
The first equivalence holds because of the commutativity and associativity of $T$. The second one holds because $T$ is left-continuous. The third one is a consequence of the $T$-transitivity of $\widetilde{R}$ while the fourth equivalence follows from the residuation property. 
\end{proof}
Please note that the $T$-disjointness is characterised by the above proposition only for granules of opposite types, i.e., granules w.r.t. relations $\widetilde{R}$ and $\widetilde{R}^{-1}$ respectively. 

\begin{proposition}
\label{prop:disjoint_complement}
Let $(T, I, N)$ be an IMTL triplet. Then, fuzzy set $A$ is granularly representable w.r.t. $\widetilde{R}$ if and only if the granules from $A$ (w.r.t. $\widetilde{R}$) and $coA$ (w.r.t. $\widetilde{R}^{-1}$) are disjoint.
\end{proposition}
\begin{proof}
We have the following equivalences.
\begin{align*}
    A(u) \geq T(\widetilde{R}(u,v), A(v)) &\Leftrightarrow \widetilde{R}(u,v) \leq I(A(v), A(u)) \\
    & \Leftrightarrow N(\widetilde{R}(u,v)) \geq N(I(A(v), A(u))) \\
    & \Leftrightarrow N(\widetilde{R}(u,v)) \geq T(A(u), N(A(v))). \\
\end{align*}
The last equivalence holds because of (\ref{eq:implicator_negator_property}), while the second one follows from the fact that $N$ is a decreasing function.
The equivalences state that the granular representability of $A$ is equivalent to the $T$-disjointness condition of granules from $A$ and $coA$, as formulated in Proposition \ref{prop: t disjoint}.
\end{proof}

\begin{corollary}
Let $(T, I, N)$ be an IMTL triplet. The granules from $\underline{\text{apr}}_{R}^{\min, I}(A)$ and $\overline{\text{apr}}_{R}^{\max, T}(coA)$ are disjoint (analogously, the granules from $\underline{\text{apr}}_{R}^{\min, I}(coA)$ and $\overline{\text{apr}}_{R}^{\max, T}(A)$ are disjoint too).
\end{corollary}

\begin{proof}
The result holds from the duality property of the lower and upper approximations (\ref{eq:duality}).
\end{proof}

Next, we examine a pair of granules $\widetilde{R}^+_{\lambda_1}$ and $\widetilde{R}^-_{\lambda_2}$ that are not only disjoint, but are adjacent to each other. In other words, if their parameters are $\lambda_1$ and $\lambda_2$, then adding any $\epsilon$ to either $\lambda_1$ or $\lambda_2$ will cause the granules to overlap. For fixed $\lambda_1$, the largest $\lambda_2$ for which the granules are still disjoint is:
$$
\lambda^{\max}_2 = \sup \{\lambda; T(\lambda_1, \lambda) \leq N(\widetilde{R}(v,u))\} = I(\lambda_1, N(\widetilde{R}(v,u))).
$$
Obviously, 
$$
T(\lambda_1, \lambda^{\max}_2) = T(\lambda_1, I(\lambda_1, N(\widetilde{R}(v,u)))) \leq N(\widetilde{R}(v,u))
$$ 
due to the modus ponens property (\ref{eq:modus_ponens}). 
\begin{definition}
\label{def:adjecency_relation}
Granule $\widetilde{R}^-_{\lambda_2}(v)$ is adjacent to granule $\widetilde{R}^+_{\lambda_1}(u)$ if
$$
\lambda_1 = I(\lambda_2, N(\widetilde{R}(v,u))),
$$
while $\widetilde{R}^+_{\lambda_1}(u)$ is adjacent to $\widetilde{R}^-_{\lambda_2}(v)$ if 
$$
\lambda_2 = I(\lambda_1, N(\widetilde{R}(v,u))).
$$
We call such defined relationship among granules the adjacency relation.
\end{definition}
\begin{proposition}
\label{prop:adj_to_1}
Every granule is adjacent to all granules with parameter 1, under the assumption that they are disjoint.
\end{proposition}
\begin{proof}
If $\lambda_1 = 1$, from the $T$-disjointness property we have:
$$
T(1, \lambda_2) \leq N(\widetilde{R}(v,u)) \Leftrightarrow 1 \leq I(\lambda_2, N(\widetilde{R}(v,u))) \Rightarrow 1 = I(\lambda_2, N(\widetilde{R}(v,u))).
$$
\end{proof}
From the proof of Proposition \ref{prop:adj_to_1}, we may conclude that granule $\widetilde{R}^+_{\lambda_1}(u)$ for $\lambda_1 = 1$ is adjacent to $\widetilde{R}^-_{\lambda_2}(v)$ if and only if $\lambda_2 = N(\widetilde{R}(v,u))$.

The previous reasoning also reveals that the adjacency relation is not necessarily symmetric.
\begin{proposition}
For parameters $\lambda_1$ and $\lambda_2$ that are smaller than 1 and for continuous $t$-norm $T$ from the IMTL triplet $(T, I, N)$, we have that the adjacency relation is symmetric. In other words, if $\lambda_1 = I(\lambda_2, N(\widetilde{R}(v,u)))$, then also $\lambda_2 = I(\lambda_1, N(\widetilde{R}(v,u)))$.
\end{proposition}
\begin{proof}
Ordering property (\ref{eq:ordering_property}) implies that if $\lambda_1 < 1$, then also $\lambda_2 > N(\widetilde{R}(v,u))$.
Using the strong max-definability (\ref{eq:strong_max_definability}), we have that
$$
\lambda_2 = I(I(\lambda_2, N(\widetilde{R}(v,u))), N(\widetilde{R}(v,u))) = I(\lambda_1, N(\widetilde{R}(v,u))).
$$
\end{proof}
\begin{figure}[!htb]
    \centering
    
    \includegraphics[width = .48\textwidth]{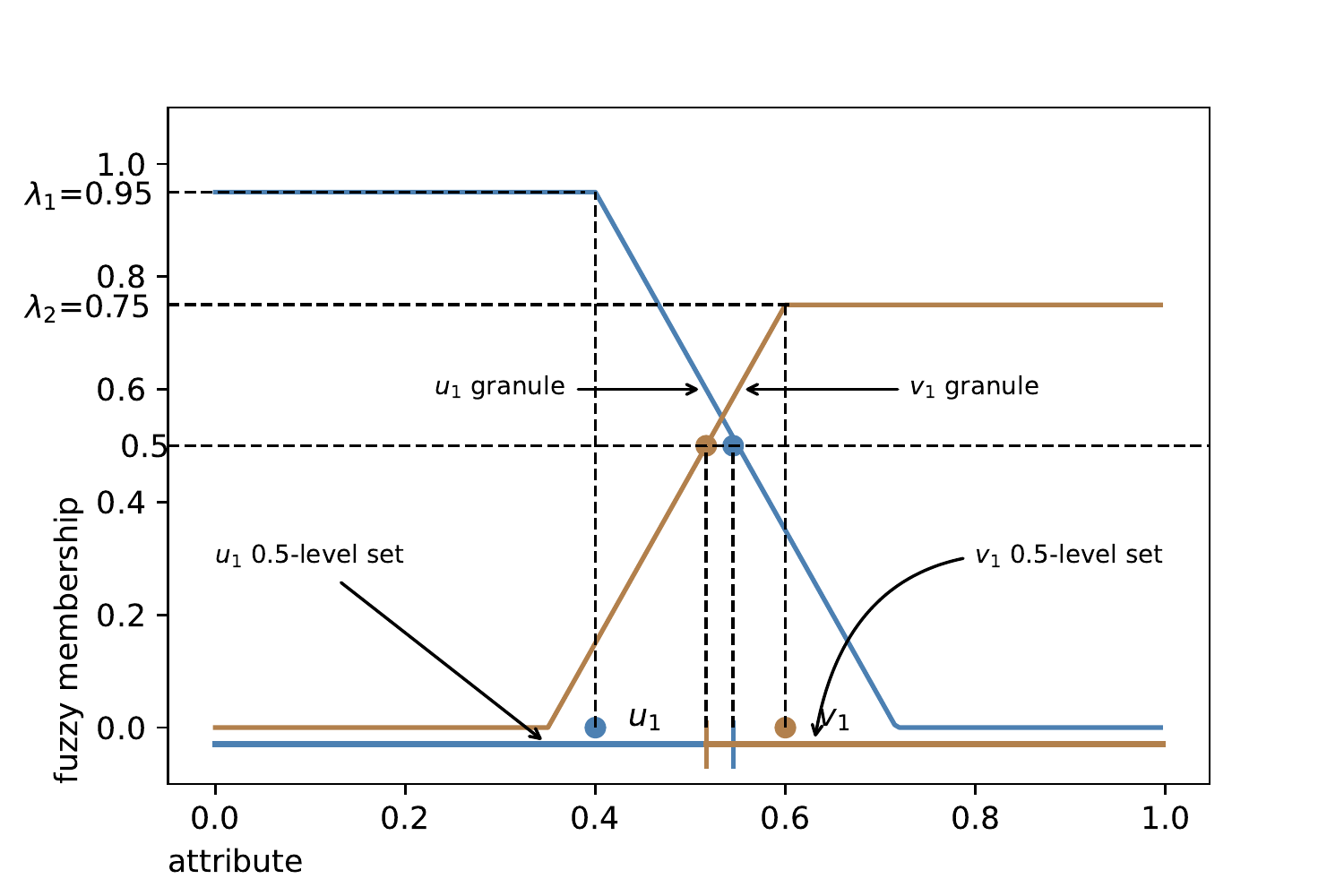}
    \includegraphics[width = .48\textwidth]{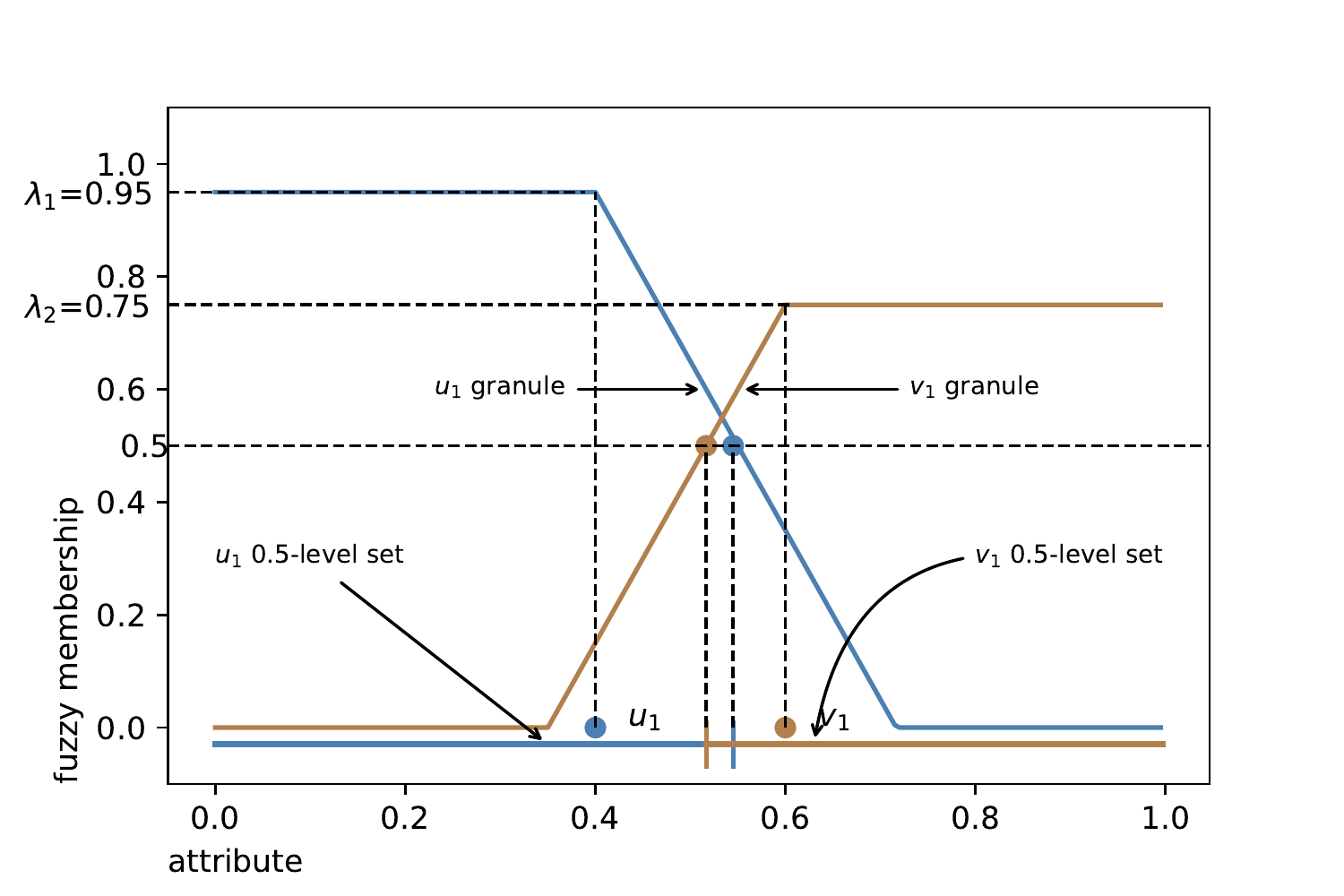}
    \includegraphics[width = .48\textwidth]{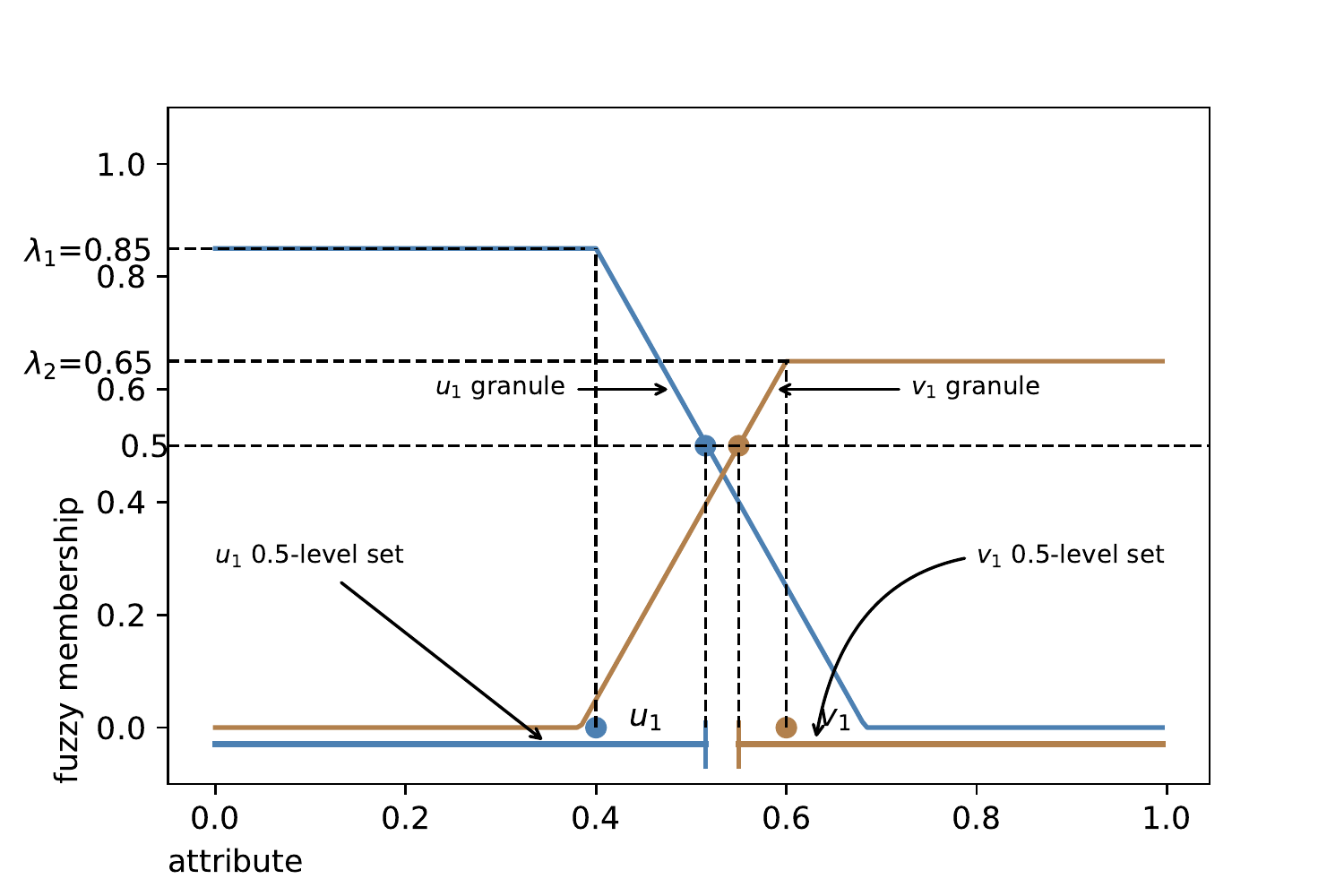}
    \includegraphics[width = .48\textwidth]{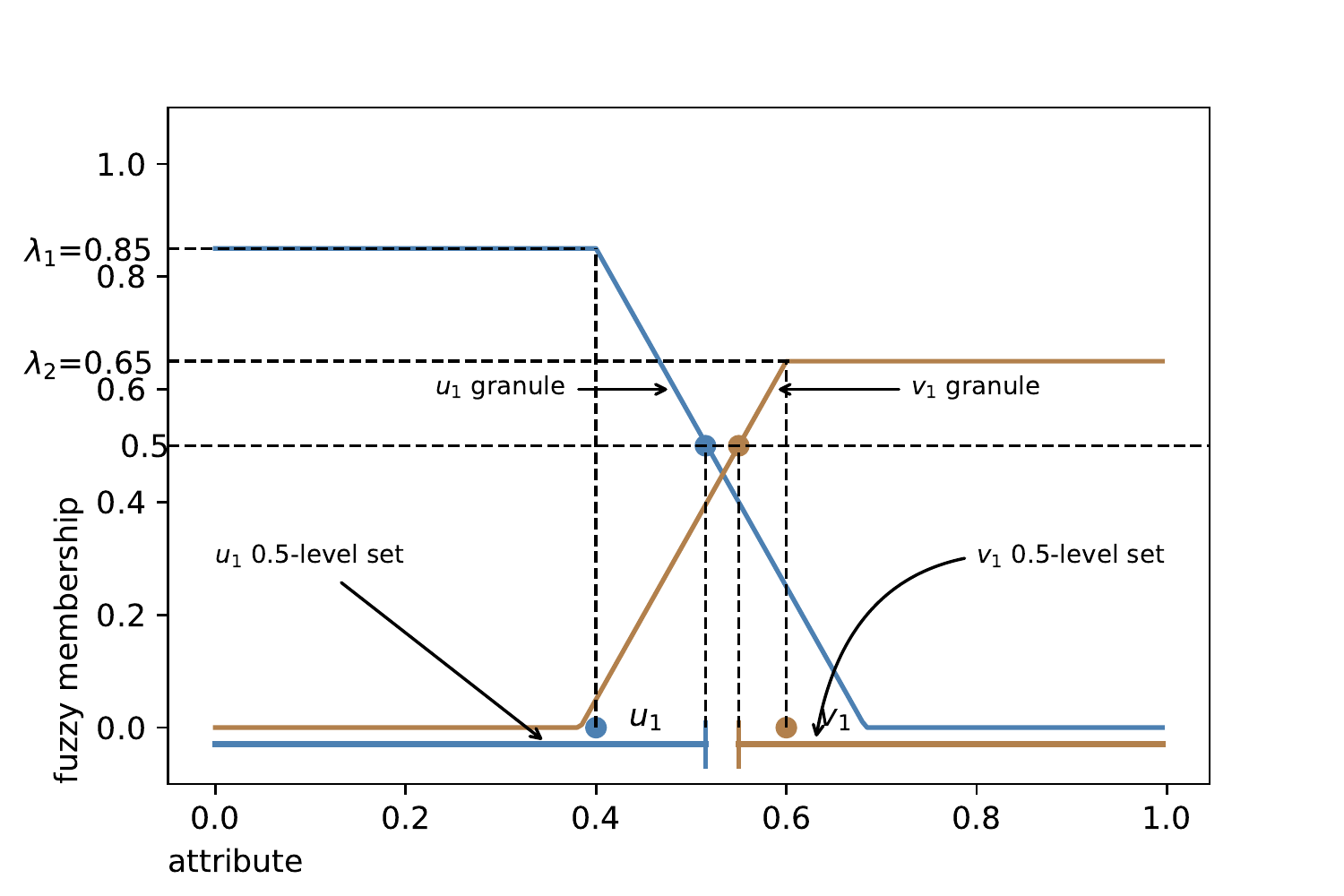}
    \includegraphics[width = .48\textwidth]{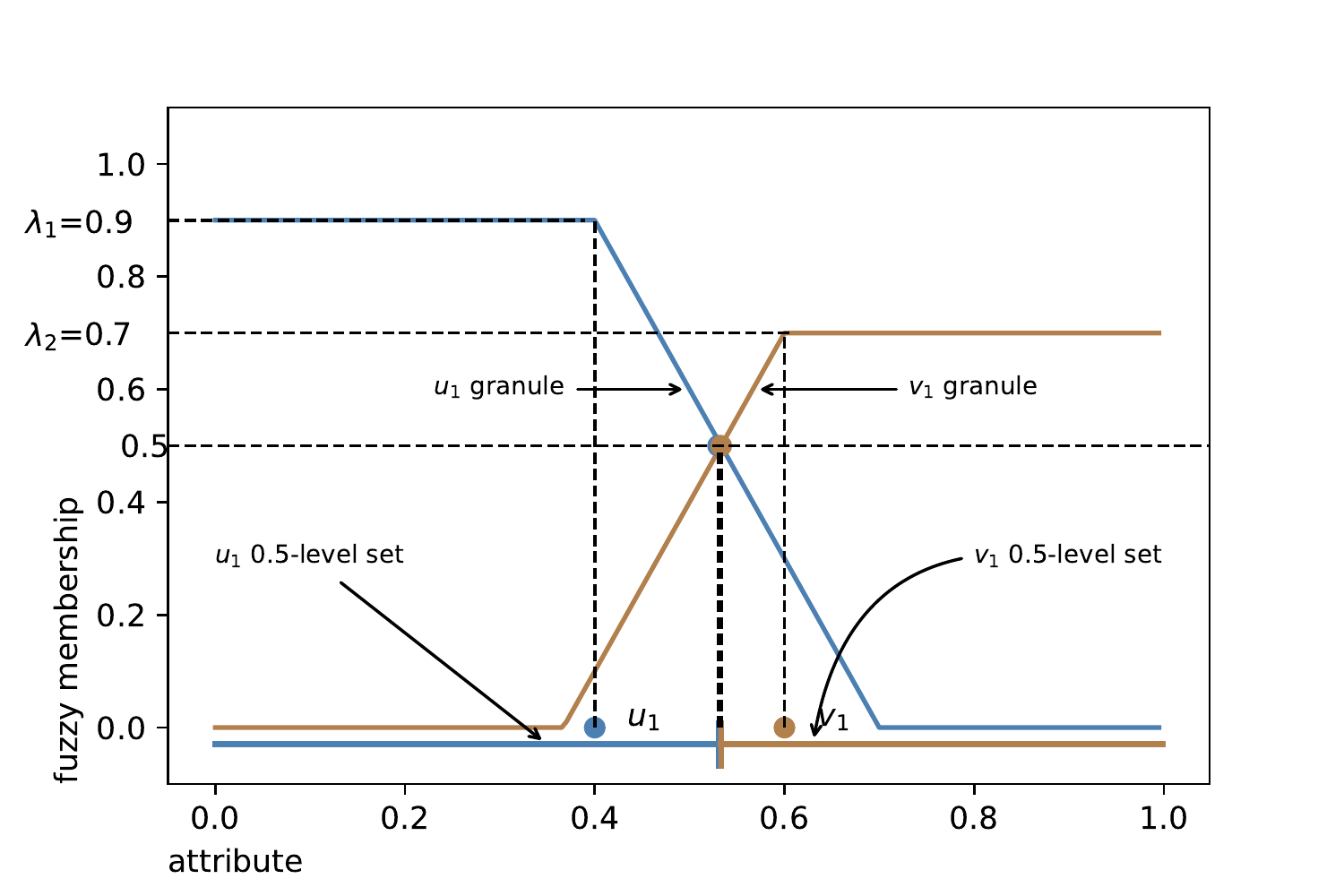}
    \includegraphics[width = .48\textwidth]{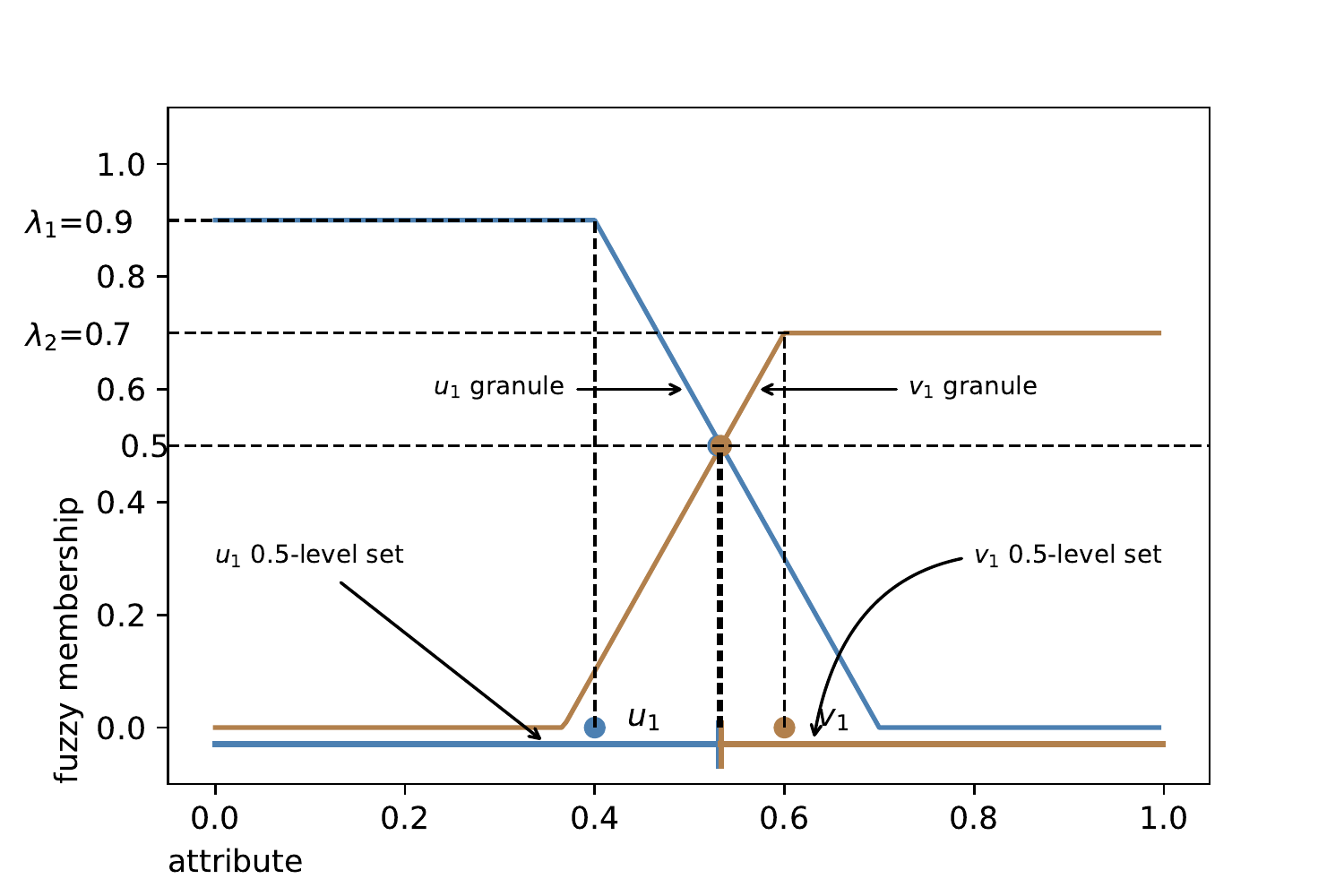}
    \caption{Granules in one dimension}
    \label{fig:granules_relation_1d}
\end{figure}{}
\begin{example}
Figures \ref{fig:granules_relation_1d} and \ref{fig:granules_relation} illustrate different relationships between granules, in one and two dimensions respectively. In Figure \ref{fig:granules_relation_1d}, objects are represented using one condition attribute $q$ whose range is 1. There are two objects $u_1$ and $v_1$ with respective attribute values $0.4$ and $0.6$. Their granules $R_{\lambda_1}^-(u)$ and $R_{\lambda_2}^-(v)$ are formed based on a $T$-preorder relation (left side of the figure) and $T$-equivalence relation (right side of the figure), parameter value $\gamma = 3$ and the Łukasiewicz $t$-norm. We vary parameters $\lambda_1$ and $\lambda_2$ in order to represent different relationship among two granules. We depict the fuzzy granules together with their 0.5-level sets. In the upper two images, the values of parameters are $\lambda_1=0.95$ and $\lambda_2=0.75$ which leads to overlapping granules (i.e., they are not $T$-disjoint). In the two images in the middle, the values of parameters are $\lambda_1=0.85$ and $\lambda_2=0.65$ which leads to $T$-disjoint granules, while in the lower two images, the values of parameters are $\lambda_1=0.9$ and $\lambda_2=0.7$ which leads to adjacent granules (here, the adjacency relation is symmetric). It is easy to verify that in this case, the 0.5-level sets follow the relation between granules, i.e., if the granules overlap, then the level sets overlap, if the granules are $T$-disjoint, then the level sets are disjoint and if the granules are adjacent, then the level sets have one common point. 

\begin{figure}[!htb]
    \centering
    \includegraphics[width = .48\textwidth]{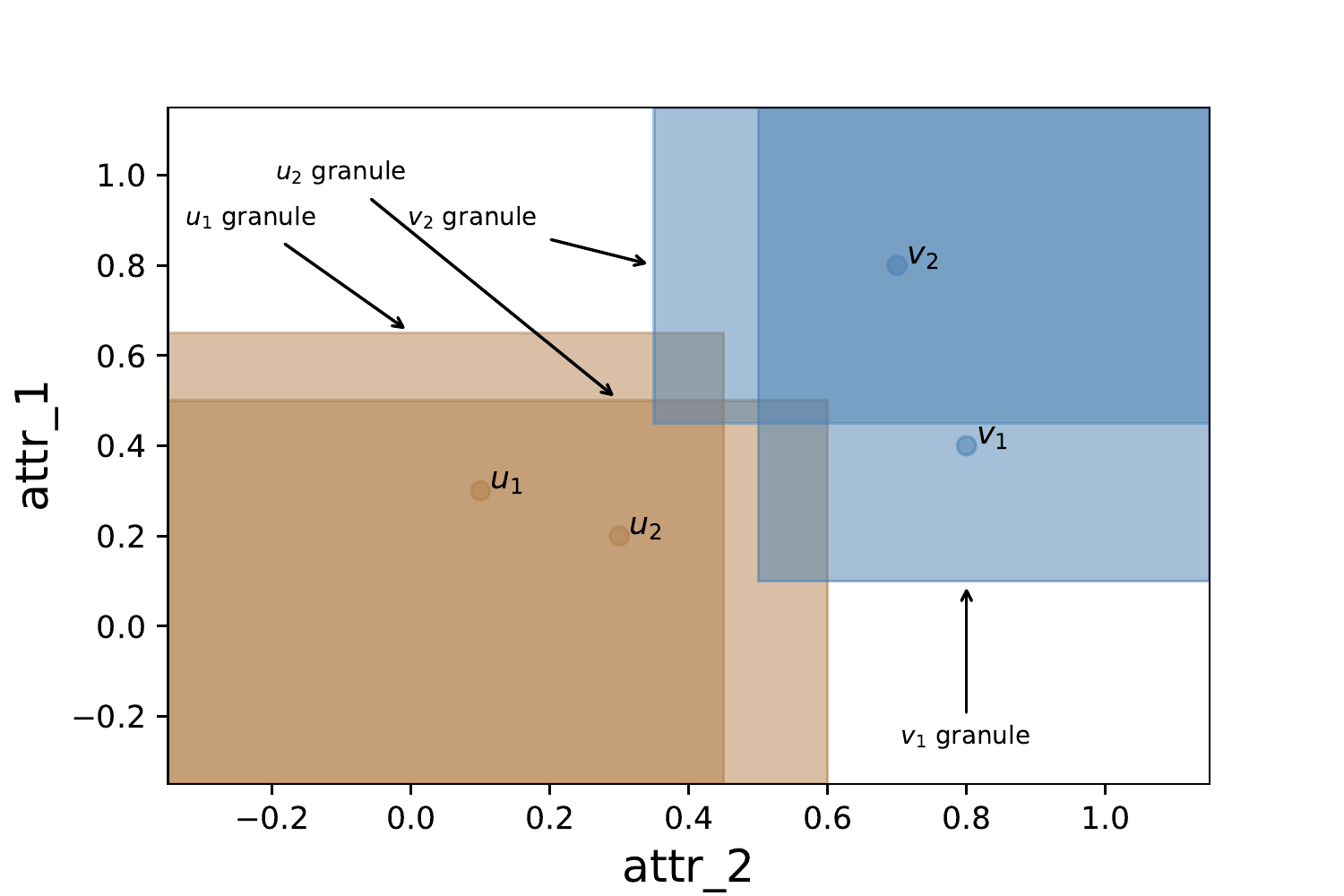}
    \includegraphics[width = .48\textwidth]{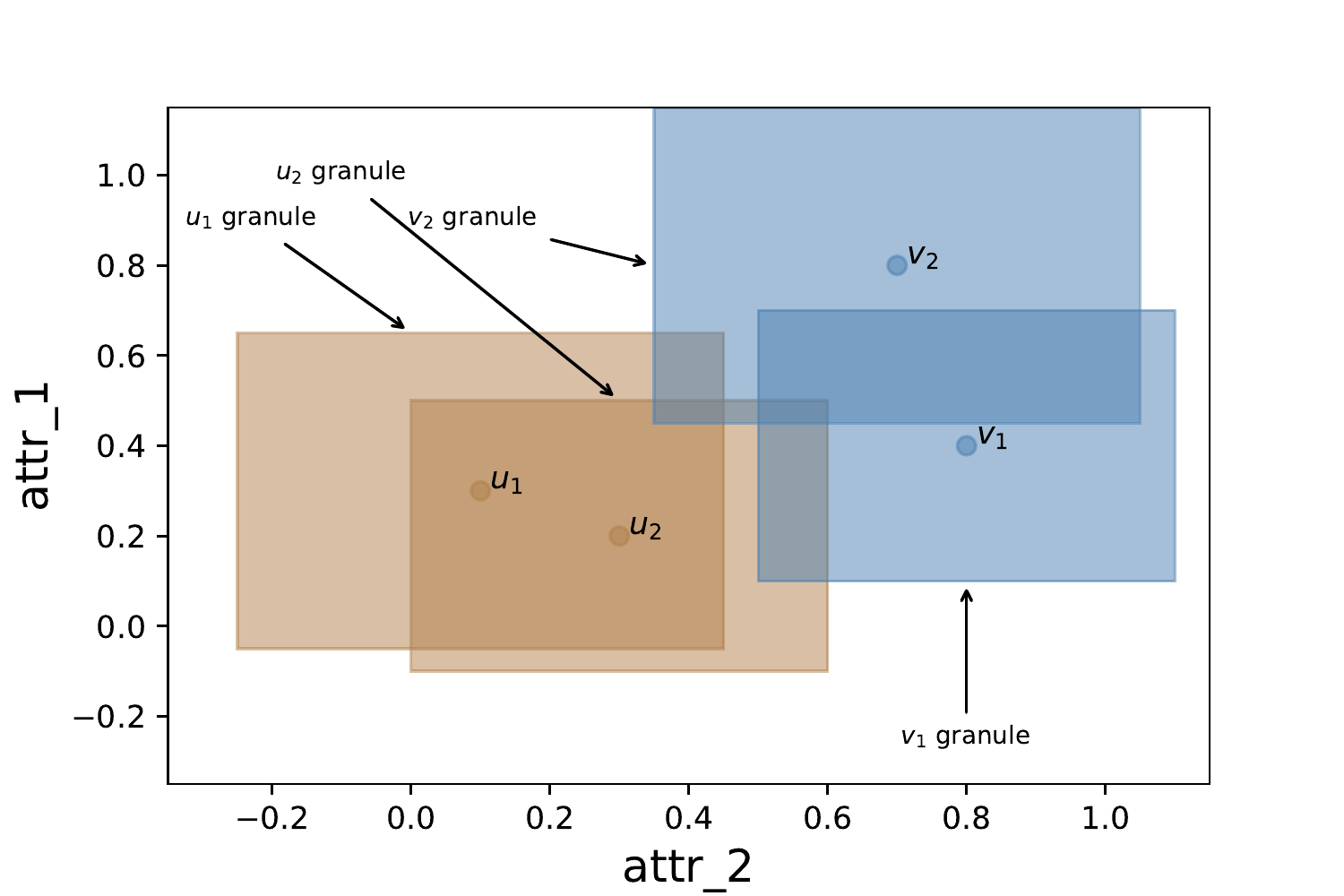}
    \includegraphics[width = .48\textwidth]{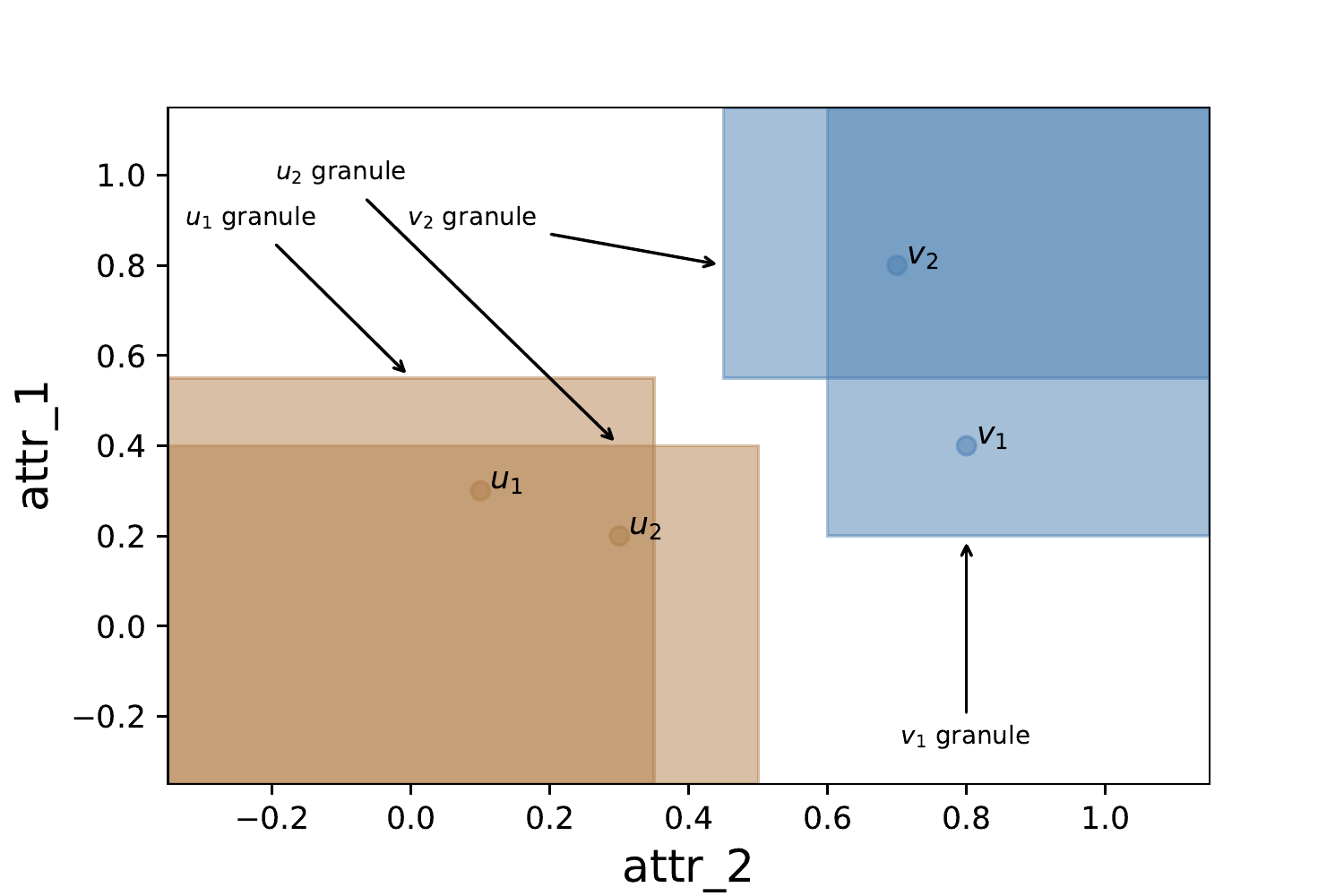}
    \includegraphics[width = .48\textwidth]{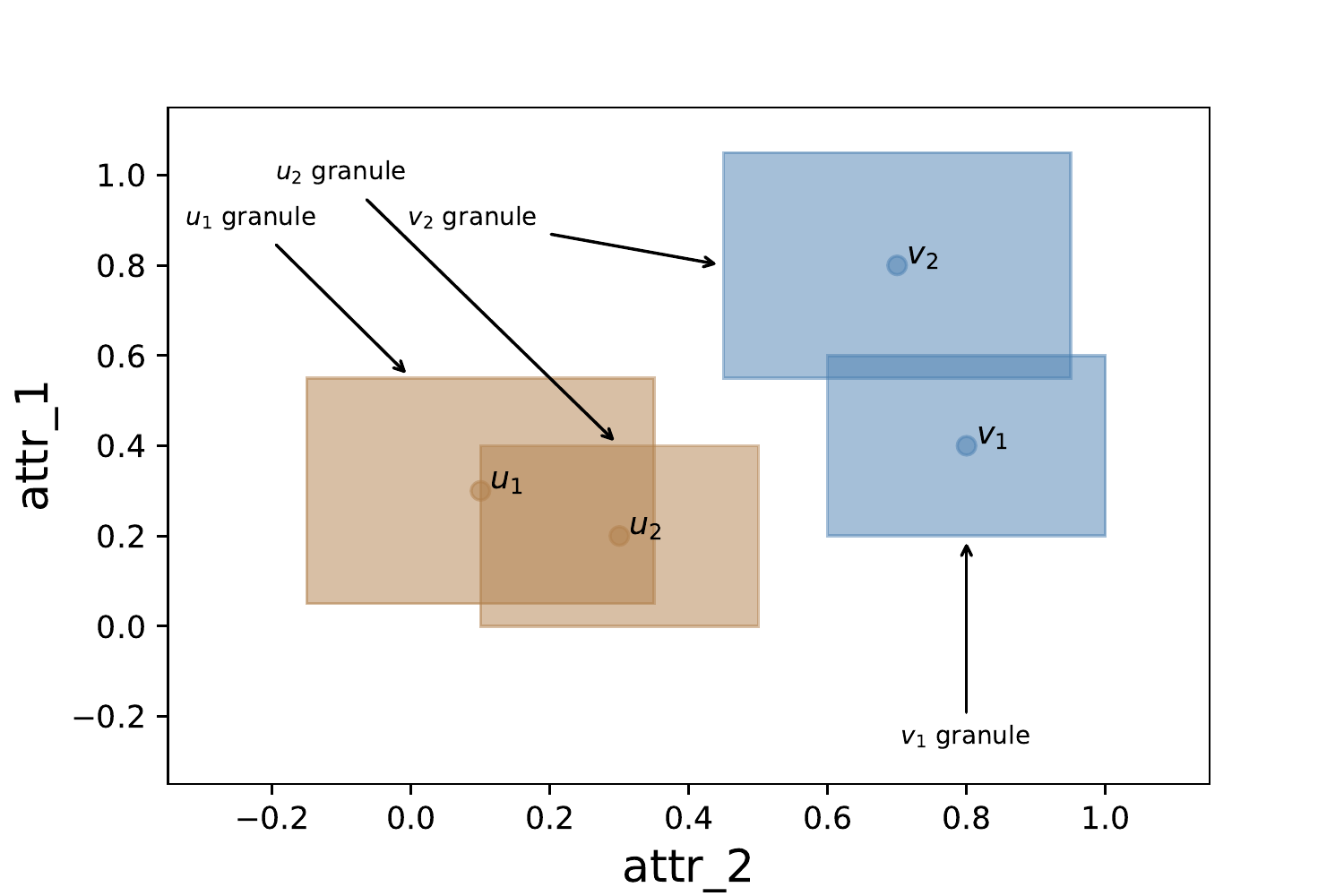}
    \includegraphics[width = .48\textwidth]{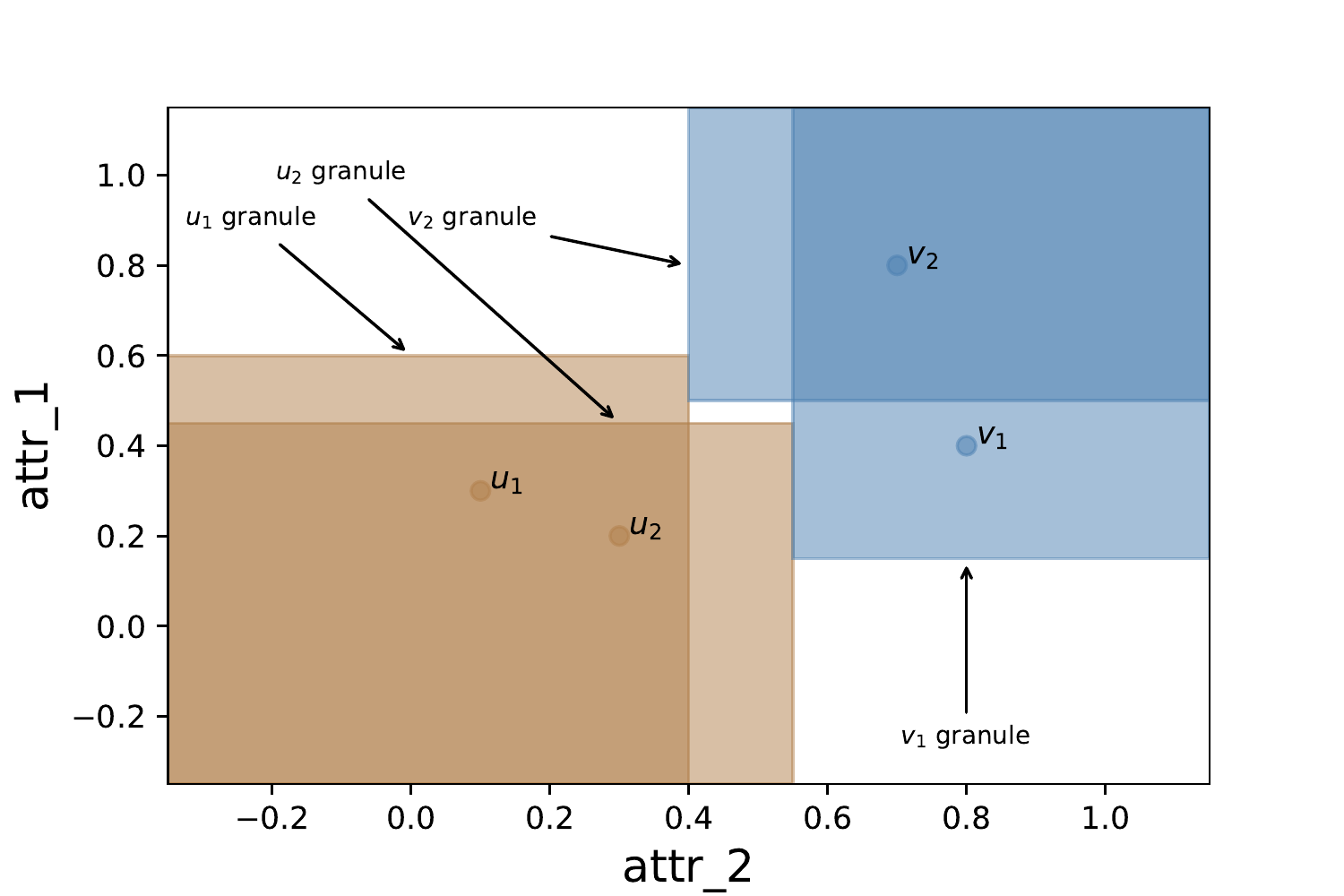}
    \includegraphics[width = .48\textwidth]{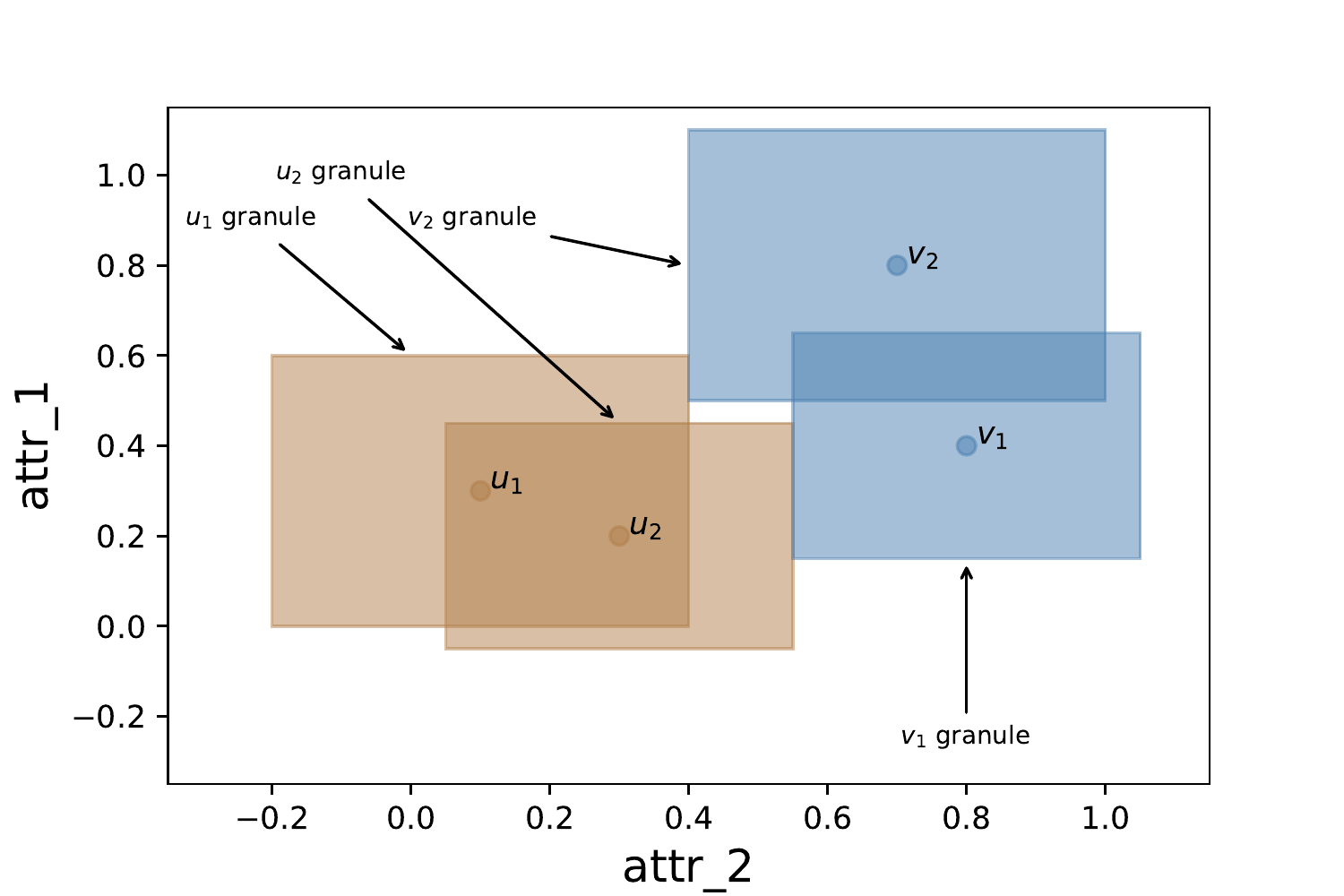}
    \caption{Granules in two dimensions}
    \label{fig:granules_relation}
\end{figure}{}
The use of 0.5-level sets is in particular useful to visualize the granules in the case of two dimensions. In Figure \ref{fig:granules_relation}, we have four objects from two classes, described by two condition attributes; $u_1$ and $u_2$ are from one class and $v_1$ and $v_2$ are from the other one. We illustrate the relationship of granules from different classes. The granules are formed using a $T$-preorder relation (left side of the figure) and $T$-equivalence relation (right side of the figure). While the granules in one dimension had triangular ($T$-equivalence) or "half-triangular" shape ($T$-preorder), in two dimensions they have pyramidal ($T$-equivalence) or "half-pyramidal" ($T$-preorder) shape. However, for the purpose of the visualisation in 2 dimensions of the 3-dimensional granules, we use 0.5-level sets of the granules. The level sets for the $T$-preorder relation take the form of quarter-plane (left images) and the rectangular form (right images) for the $T$-equivalence relation. Again, we can distinguish granules from different classes that overlap (upper two images), that are disjoint (two images in the middle) and that are adjacent (lower two images). In this case, the level sets of adjacent granules share an edge.
\end{example}

In Proposition \ref{prop:disjoint_complement} we showed that granules from $A$ and $coA$ are $T$-disjoint for granularly representable set $A$. Now we examine in which cases some of them are also adjacent. If $\widetilde{R}^-_{coA(v)}(v)$ is adjacent to $\widetilde{R}^+_{A(u)}(u)$ for $u, v \in U$, we have that
\begin{equation}
\label{eq:one_way_adjacent}
    A(u) = I(N(A(v)), N(\widetilde{R}(v, u))) \Leftrightarrow A(u) = I(\widetilde{R}(v, u), A(v)),
\end{equation}
where the equivalence holds because of (\ref{eq:implicator_negator_property3}).
If $\widetilde{R}^+_{A(u)}(u)$ is adjacent to $\widetilde{R}^-_{coA(v)}(v)$ then
\begin{align}
\label{eq:other_way_adjacent}
\begin{split}
    N(A(v)) = I(A(u), N(\widetilde{R}(u,v))) &\Leftrightarrow A(v) = N(I(A(u), N(\widetilde{R}(v,u)))) \\
    & \Leftrightarrow A(v) = T(\widetilde{R}(v,u), A(u)),
\end{split}
\end{align}
where the second equivalence holds because of (\ref{eq:implicator_negator_property2}).

\subsection{Application to granular approximations}
The next lemma, theorem, and corollary investigate the adjacency relationship of granules from the granular approximations, i.e., from the solutions of optimization problem (\ref{eq:general_optimization}).

\begin{lemma}
\label{lemma:small_self_adjacency}
Let loss function $L$ be of $\lor$-type and let $\hat{A}$ be a solution of optimization problem ($\ref{eq:general_optimization}$). If $\hat{A}(u) > A(u)$, then it holds that
$$
\hat{A}(u) = \max \{T(\widetilde{R}(u,v), \hat{A}(v)); v \in U, v \neq u\},
$$
while if $\hat{A}(u) < A(u)$, then it holds that
$$
\hat{A}(u) = \min \{I(\widetilde{R}(v,u), \hat{A}(v)); v \in U, v \neq u\} 
$$
\end{lemma}
\begin{proof}
Let $\alpha_u = \max \{T(\widetilde{R}(u,v), \hat{A}(v)); v \in U, v \neq u\}$ for some $u, \hat{A}(u) > A(u)$. If the first condition of the theorem is not satisfied, then from the granular representability it holds that $\alpha_u < \hat{A(u)}$. Replacing $\hat{A(u)}$ with $\max(\alpha_u, A(u))$ leads to a solution that is also granularly representable (easily verifiable) and that ensures a smaller value of the objective function since the loss function $L$ is of $\lor$-type. That is a contradiction. 

Now, let $\alpha_u = \min \{I(\widetilde{R}(v,u), \hat{A}(v)); v \in U, v \neq u\} $ for some $u, \hat{A}(u) < A(u)$. If the second condition of the theorem is not satisfied, then from the granular representability it holds that $\alpha_u > \hat{A(u)}$. Replacing $\hat{A(u)}$ with $\min(\alpha_u, A(u))$ leads to a solution that is also granularly representable (easily verifiable) and that ensures a smaller value of the objective function since the loss function $L$ is of $\lor$-type. That is a contradiction. 
\end{proof}

\begin{theorem}
\label{thm:big_self_adjacency}
Let loss function $L$ be of $\lor$-type and let $\hat{A}$ be a solution of optimization problem ($\ref{eq:general_optimization}$). We define three sets:
\begin{itemize}
    \item $U^- = \{u \in U; \hat{A}(u) < A(u)\}$,
    \item $U^0 = \{u \in U; \hat{A}(u) = A(u)\}$,
    \item $U^+ = \{u \in U; \hat{A}(u) > A(u)\}$.
\end{itemize}
It holds that
\begin{equation}
\label{eq:fuzzy_adjecency1}
    \hat{A}(u) = \max \{T(\widetilde{R}(u,v), \hat{A}(v)); v \in U^- \cup U^0\},
\end{equation}
for $u \in U^+$ and
\begin{equation}
\label{eq:fuzzy_adjecency2}
    \hat{A}(u) = \min \{I(\widetilde{R}(v,u), \hat{A}(v)); v \in U^+ \cup U^0\},
\end{equation}
for $u \in U^-$.
\end{theorem}
\begin{proof}
    Condition (\ref{eq:fuzzy_adjecency1}) can be reformulated as
    $$
    \forall u\in U^+,\, \exists v \in U^- \cup U_0; \hat{A}(u) = T(\widetilde{R}(u,v), \hat{A}(v)).
    $$
    Let $U_1^+ \subseteq U^+$ be a set of instances which satisfy the previous condition and let $U_2^+ =  U^+ - U_1^+$. 
    Let $u^* \in U_2^+$ be an instance with the largest $\hat{A(u)}$. From Lemma \ref{lemma:small_self_adjacency}, there is $v \in U, v \neq u^*$ such that $\hat{A}(u^*) = T(\widetilde{R}(u,v), \hat{A}(v))$.
    
    By the assumption $u^* \in U_2^+$, we have that $v \in U^+$. If $v \in U_1^+$, then there is $w \in U^- \cap U^0$ such that $\hat{A}(v) = T(\widetilde{R}(v, w), \hat{A}(w))$. We have that 
    \begin{align*}
        A(u^*) &= T(\widetilde{R}(u^*,v), \hat{A}(v)) \\
            &= T(\widetilde{R}(u^*,v), T(\widetilde{R}(v, w), \hat{A}(w))) \\
            & = T(T(\widetilde{R}(u^*,v), \widetilde{R}(v, w)), \hat{A}(w)) \\
            & \leq T(\widetilde{R}(u^*, v), \hat{A}(w)).
    \end{align*}
    The last inequality holds because of the $T$-transitivity of $\widetilde{R}$ and the monotonicity of $T$. The opposite inequality holds from the granular representability which leads to the conclusion that $ A(u^*) = T(\widetilde{R}(u^*, v), \hat{A}(w))$ which contradicts the assumption that $u^* \in U_2^+$. Hence, $v \in U^+_2$. 
    
    From $\hat{A}(u^*) = T(\widetilde{R}(u^*,v), \hat{A}(v))$, it holds $\hat{A}(v) \geq \hat{A}(u^*)$ due to (\ref{eq:t-norm_smaller_parameters}). Since $A(u^*)$ is the largest in $U_2^+$ by the assumption, then $\hat{A}(u^*) = \hat{A}(v)$. Denote with $U_3^+ \subseteq U_2^+$ instances from $U_2^+$ for which the membership degree in $\hat{A}$ is $\hat{A}(u^*)$. Every pair of instances from $U_3^+$ satisfies (\ref{eq:granular_rep}) since they have the same membership value in $\hat{A}$. Due to maximality of $\hat{A}(u)$ it holds that for $u \in U_3 ^+$ and for $v \in U - U_3 ^+$, it holds that $\hat{A}(u) > T(\widetilde{R}(u,v), \hat{A}(v))$. Denote
    $$
    \beta^+ = \max(\max \{A(u);  u \in U_3 ^+\}, \max \{ \max\{T(\widetilde{R}(u,v), \hat{A(v)}); v \in U - U_3^+ \}; u \in U_3^+ \}).
    $$
    From the assumptions above, it holds that $\beta^+ < \hat{A}(u^*)$ which implies $\beta^+ < \hat{A}(u)$ for $u \in U_3^+$. Now, let $\hat{A}^*$ be a fuzzy set where values $\hat{A}(u)$ for $u \in U_3^+$ are replaced with $\beta^+$. We observe that $\hat{A}^*$ is granularly representable since $\hat{A}^*(u)$ are pairwise equal for $u \in U^+_3$ and also for every $u \in U_3 ^+$ and for every $v \in U - U_3 ^+$ it holds that $T(\widetilde{R}(u,v), \hat{A}^*(v)) \leq \hat{A}^*(u)$ by the definition of $\beta^+$. Next, we observe that the objective value with $\hat{A}^*$ is smaller than with $\hat{A}$ because $A(u) < \hat{A}^*(u) < \hat{A}(u)$ for $u \in U^+_3$ and due to the fact that $L$ is of $\lor$-type. 
    
    Therefore, we obtained a feasible solution with a smaller objective function, which contradicts the optimality of $\hat{A}$. This contradiction implies that $U_2^+$ must be empty which is equivalent to ($\ref{eq:fuzzy_adjecency1}$).
    
    On the other hand, condition (\ref{eq:fuzzy_adjecency2}) can be reformulated as
    $$
    \forall u\in U^-,\, \exists v \in U^+ \cup U_0; \hat{A}(u) = I(\widetilde{R}(v,u), \hat{A}(v)).
    $$
    Let $U_1^- \subseteq U^-$ be a set of instances which satisfy the previous condition and let $U_2^- =  U^- - U_1^-$. 
    Let $u^* \in U_2^-$ be an instance with the smallest $\hat{A}(u)$. From Lemma \ref{lemma:small_self_adjacency}, there is $v \in U, v \neq u^*$ such that $\hat{A}(u^*) = I(\widetilde{R}(v,u), \hat{A}(v))$. 
    
    By the assumption $u^* \in U_2^-$, it holds that $v \in U^-$. Assume that $v \in U_1^-$. Then, there is $w \in U^+ \cap U^0$ such that $\hat{A}(v) = I(\widetilde{R}(w, v), \hat{A}(w))$. We have that 
    \begin{align*}
        A(u^*) &= I(\widetilde{R}(v,u^*), \hat{A}(v)) \\
            &= I(\widetilde{R}(v, u^*), I(\widetilde{R}(w, v), \hat{A}(w))) \\
            & = I(T(\widetilde{R}(v, u^*), \widetilde{R}(w, v)), \hat{A}(w)) \\
            & \geq I(\widetilde{R}( w, u^*), \hat{A}(w)).
    \end{align*}
    The last equality holds because of (\ref{eq:t_norm_implicator_property}).
    The last inequality holds because of the $T$-transitivity of $\widetilde{R}$ and the fact that $I$ is decreasing in its first argument. The opposite inequality holds from the granular representability which leads to the conclusion that $ A(u^*) = I(\widetilde{R}(w, u^*), \hat{A}(w))$ which contradicts the assumption that $u^* \in U_2^+$. Hence, $v \in U^-_2$. 
    
    From $\hat{A}(u^*) = I(\widetilde{R}(v, u^*), \hat{A}(v))$, it holds that $\hat{A}(v) \leq A(u^*)$ due to (\ref{eq:implicator_greater_second_parameter}). Since $\hat{A}(u^*)$ is the smallest by the assumption, then $\hat{A}(u^*) = \hat{A}(v)$. Denote with $U_3^- \subseteq U_2^-$ instances from $U_2^-$ that have value $\hat{A}(u^*)$. Every pair of instances from $U_3^-$ satisfy (\ref{eq:granular_rep}) since they have the same membership degree in $\hat{A}$. For every $u \in U_3 ^+$ and for every $v \in U - U_3 ^+$ it holds that $\hat{A}(u) < I(\widetilde{R}(v,u), \hat{A}(v))$. Denote
    $$
    \beta^- = \min(\min \{A(u);  u \in U_3 ^-\}, \min \{ \min\{I(\widetilde{R}(v,u), \hat{A(v)}); v \in U - U_3^+ \}; u \in U_3^+ \}).
    $$
    From the above assumption, it holds that $\beta^- > \hat{A}(u^*)$, which implies $\beta > \hat{A}(u)$ for $u \in U_3^-$. Now, let $\hat{A}^{**}$ be a fuzzy set where values $\hat{A}(u)$ for $u \in U_3^-$ are replaced with $\beta^-$. We observe that $\hat{A}^{**}$ is granularly representable since $\hat{A}^{**}(u)$ are pairwise equal for $u \in U^-_3$ and also for every $u \in U_3 ^-$ and for every $v \in U - U_3 ^-$ it holds that $I(\widetilde{R}(v,u), \hat{A}^{**}(v)) \leq \hat{A}^{**}(u)$ by the definition of $\beta^-$. Next, we observe that the objective value with $\hat{A}^{**}$ is smaller than with $\hat{A}$ because $A(u) > \hat{A}^{**}(u) > \hat{A}(u)$ for $u \in U^-_3$ and due to the fact that $L$ is of $\lor$-type. 
    
    Therefore, we obtained a feasible solution with a smaller objective function, which contradicts the optimality of $\hat{A}$. This contradiction implies that $U_2^-$ must be empty, which is equivalent to (\ref{eq:fuzzy_adjecency2}).

\end{proof}

\begin{corollary}
\label{cor:big_adjacent}
Let loss function $L$ be of $\lor$-type and let $\hat{A}$ be a solution of optimization problem ($\ref{eq:general_optimization}$) defined w.r.t. an IMTL triplet $(T,I,N)$. Let $U^-, U^0, U^+$ be defined as in Theorem \ref{thm:big_self_adjacency}. Then, the following holds.
\begin{itemize}
    \item For all $ u\in U^+$, there is $ v \in U^- \cup U_0$ such that $R^+_{\hat{A}(v)}$ is adjacent to $R^-_{co\hat{A}(u)}$.
    \item For all $ u \in U^-$, there is $v \in U^+ \cup U_0$ such that $R^-_{co\hat{A}(v)}$ is adjacent to $R^+_{\hat{A}(u)}$.
\end{itemize}
\end{corollary}

\begin{proof}
The corollary is a direct consequence of Theorem \ref{thm:big_self_adjacency} and equations (\ref{eq:one_way_adjacent}) and (\ref{eq:other_way_adjacent}).
\end{proof}

Note that Theorem \ref{thm:big_self_adjacency} does not require for residual triplet $(T, I, N)$ to be an IMTL triplet, hence it can lead to more general results that are not related to the granular adjacency relationships.

\section{Case of a classification problem}
\label{sec:case_of_a_classification_problem}

First, we consider a binary classification problem, i.e., we distinguish two classes in $U$: $A$ and $coA$ which are now crisp (ordinary) sets. Notations $A$ and $coA$ will be also used for the fuzzy sets that encode the corresponding decision class, i.e., $A(u) = 1$ if $u \in A$ while $A(u) = 0$ if $u \in coA$.

\begin{proposition}
\label{eq:classification_adjacency}
Let loss function $L$ be of $\lor$-type and let $\hat{A}$ be a granular approximation of a crisp set $A$ w.r.t. an IMTL triplet $(T,I,N)$. Then, the following holds.
\begin{itemize}
    \item For all $ u\in A$, there is $ v \in coA$ such that $R^-_{co\hat{A}(v)}$ is adjacent to $R^+_{\hat{A}(u)}$. 
    \item For all $ u \in coA$, there is $v \in A$ such that $R^+_{\hat{A}(v)}$ is adjacent to $R^-_{co\hat{A}(u)}$.
\end{itemize}

\end{proposition}
\begin{proof}
Let $U^+$, $U^0$, $U^-$ be the sets defined in Theorem \ref{thm:big_self_adjacency}. Obviously, it holds that $U^- \subseteq A$ and $U^+ \subseteq coA$. We prove the first part of the proposition, while the second part holds by analogy.

Let $u \in A$. If $u \in A - U^- \Leftrightarrow \hat{A}(u) = 1$, then from Proposition \ref{prop:adj_to_1} we have that for all $v \in coA$, $R^-_{co\hat{A}(v)}$ is adjacent to $R^+_{\hat{A}(u)}$. If $u \in U^-$, then from Corollary \ref{cor:big_adjacent}, there is $v \in U^0 \cup U^+$ such that $R^-_{co\hat{A}(v)}$ is adjacent to $R^+_{\hat{A}(u)}$. If $v \in U^+$, then also $v \in coA$ since $U^+ \in coA$. If $v \in U^0$, then either $\hat{A}(v)=0$ or $\hat{A}(v) = 1$. If $\hat{A}(v) = 1$ then the fact that $R^-_{co\hat{A}(v)}$ is adjacent to $R^+_{\hat{A}(u)}$ is equivalent to 
$$
\hat{A}(u) = I(N(\hat{A}(v)), N(\widetilde{R}(v,u))) = I(0, N(\widetilde{R}(v,u))) = 1,
$$
which contradicts the assumption that $u \in U^-$. The last equality holds because of the ordering property (\ref{eq:ordering_property}). Therefore, it holds that $\hat{A}(v)=0$ which implies $v \in coA$. This proves the first part of the proposition. 
\end{proof}

For a solution $\hat{A}$ of optimization problem (\ref{eq:general_optimization}), we have that $\hat{A}(u)$ for $u \in U$ represents the degree up to which $u$ belongs to decision class $A$, while $co\hat{A}(u)$ represents the degree up to which $u$ belongs to decision class $coA$. We may be interested to calculate only the degrees $\hat{A}(u)$ for $u \in A$ and degrees $co\hat{A}(u)$ for $u \in coA$. 

\begin{proposition}
Denote $\beta_u = \hat{A}(u)$ for $u \in A$ and $\beta_u = N(\hat{A}(u))$ for $u \in coA$. Let $L$ be of $\lor$-type and $N$-dual preserving and symmetric. Then, in the classification case, problem (\ref{eq:general_optimization}) is equivalent to
\begin{equation}
\label{eq:disjoint_binary_optimization}
\begin{aligned}
&\text{minimize}  && \displaystyle\sum_{u \in U} L(1, \beta_u) \\
&\text{subject to} && T(\beta_u, \beta_v) \leq N(\widetilde{R}(v,u)),  \quad   u \in A, v \in coA\\
  &              &&0 \leq \beta_u \leq 1, \quad u \in U.
\end{aligned}
\end{equation}
\end{proposition}
\begin{proof}
With the new notation and for $L$ being $N$-duality preserving, the objective function of (\ref{eq:general_optimization}) becomes:
$$
\sum_{u \in A} L(1, \beta_u) + \sum_{u \in coA} L(0, N(\beta_u)) = \sum_{u \in A} L(1, \beta_u) + \sum_{u \in coA} L(\beta_u, 1) = \sum_{u \in U} L(1, \beta_u).
$$
The granularity constraints from (\ref{eq:general_optimization}) are now divided into 3 groups:
\begin{itemize}
    \item Granularity constraints for pairs of objects $u,v \in A$: 
    $$
    \beta_u \geq T(\widetilde{R}(u,v), \beta_v).
    $$
    \item Granularity constraints for pairs of objects $u,v \in coA$:
    $$
    N(\beta_u) \geq T(\widetilde{R}(u,v), N(\beta_v)) \Leftrightarrow \beta_v \geq T(\widetilde{R}(v, u), \beta_u).
    $$
    \item Granularity constraints for pairs of objects $u \in A, v \in coA$. In that case, the granularity condition is expressed using $T$-disjointness (according to Proposition \ref{prop:disjoint_complement}) as:
    $$
    T(\beta_u, \beta_v) \leq N(\widetilde{R}(v,u)).
    $$
\end{itemize}
The goal is to show that the first two groups of constraints are redundant.
We first prove that the adjacency from Proposition \ref{eq:classification_adjacency} still holds in problem (\ref{eq:disjoint_binary_optimization}), i.e., for every $u \in A$, there is $v \in coA$ such that $\beta_u = I(\beta_v, N(\widetilde{R}(v,u)))$ and that for all $v \in coA$, there is $u \in A$ such that $\beta_v = I(\beta_u, N(\widetilde{R}(v,u)))$. Using the residuation property, we have that
$$
T(\beta_u, \beta_v) \leq N(\widetilde{R}(v,u)) \Leftrightarrow \beta_u \leq I(\beta_v, N(\widetilde{R}(v,u))).
$$ 
If for some $u$ and for all $v$ it holds that $\beta_u < I(\beta_v, N(\widetilde{R}(v,u)))$, then there is $\epsilon > 0$ such that replacing $\beta_u$ with $\beta_u + \epsilon$ leads to a smaller objective function since the loss function is of $\lor$-type. This leads to a contradiction with the assumption that $\beta$ is optimal. 
Again, using the residuation property we have
$$
T(\beta_u, \beta_v) \leq N(\widetilde{R}(v,u)) \Leftrightarrow \beta_v \leq I(\beta_u, N(\widetilde{R}(v,u))).
$$
Using the same arguments as above, we get the second equality. 

Next, we prove that the granularity criteria for $\beta_u$ and $\beta_v$ for $u,v \in A$ are satisfied. Let $w \in coA$ such that $\beta_u = I(\beta_w, N(\widetilde{R}(w,u)))$. From the constraints, it holds that $\beta_v \leq I(\beta_w, N(\widetilde{R}(w,v))) \Leftrightarrow \beta_w \leq I(\beta_v, N(\widetilde{R}(w,v)))$. Then, we have that
\begin{align*}
    \beta_u &= I(\beta_w, N(\widetilde{R}(w,u))) \\
    & \geq I(I(\beta_v, N(\widetilde{R}(w,v))),N(\widetilde{R}(w,u))) \\
    & \geq T(\beta_v, I(N(\widetilde{R}(w,v)), N(\widetilde{R}(w,u)))) \\
    & = T(\beta_v, I(\widetilde{R}(w,u), \widetilde{R}(w,v))) \\
    & \geq T(\beta_v, \widetilde{R}(u,v)),
\end{align*}
which is exactly the granularity condition for $\beta_u$ and $\beta_v$. Now, let $u, v \in coA$ and let $w \in A$ be such that $\beta_u = I(\beta_w, N(\widetilde{R}(u,w)))$.  From the constraints, it holds that $\beta_w \leq I(\beta_v, N(\widetilde{R}(v,w)))$. Using a similar reasoning as above, we conclude that the granularity condition is also satisfied for $\beta_u$ and $\beta_v$ when $u,v \in coA$.
Since the granularity constraints for pairs of objects from the same class are a consequence of the $T$-disjointness constraints, they can be omitted in the optimization problem. 
\end{proof}

From now on, we assume that $\widetilde{R}(u,v)$ is also a symmetric relation, i.e., it is a $T$-equivalence. In such case, the granules in $A$ and $coA$ are of the same type. We now consider crisp equivalence relation $S$ on $U$ defined as $S(u,v) = 1$ if $u$ and $v$ are from the same decision class, and $S(u,v)=0$ otherwise. If $u$ and $v$ are from different decision classes then $I(\widetilde{R}(u,v), S(u,v)) = N(\widetilde{R}(u,v))$, while $I(\widetilde{R}(u,v), S(u,v)) = 1$ otherwise. With relation $S$, the $T$-disjointness constraints from (\ref{eq:disjoint_binary_optimization}) may be reformulated as
\begin{align}
\label{eq:T-disjoint_condition}
    T(\beta_u, \beta_v) \leq I(\widetilde{R}(u,v), S(u,v))), \quad u, v \in U.
\end{align}

Here, we need to note that $S$, as an equivalence relation, can distinguish among more than two decision classes. In other words, $S$ can be used to model multi-class classification problems. Bearing this in mind, a multi-class extension of problem (\ref{eq:disjoint_binary_optimization}) can be formulated as:
\begin{equation}
\label{eq:disjoint_optimization}
\begin{aligned}
&\text{minimize}  && \displaystyle\sum_{u \in U} L(1, \beta_u) \\
&\text{subject to} && T(\beta_u, \beta_v) \leq I(\widetilde{R}(u,v), S(u,v))),  \quad   u , v \in U\\
  &              &&0 \leq \beta_u \leq 1, \quad u \in U.
\end{aligned}
\end{equation}
We name the result of problem (\ref{eq:disjoint_optimization}) as \textit{multi-class granular approximation}.

We need to stress that while the binary classification problem (\ref{eq:disjoint_binary_optimization}) with a $T$-preorder relation is suitable for the binary monotone classification problems, i.e., classification problems where there exists a monotone relationship between condition attributes and a decision attribute, the problem (\ref{eq:disjoint_optimization}) with a $T$-equivalence relation is suitable for ordinary classification problems i.e., problems where such monotone relationship cannot be inferred.

\section{Calculation}
\label{sec:calculation}
In this section, we use the notation: $M(u,v) = I(\widetilde{R}(u,v), S(u,v))$.
We start with an important property.
\begin{proposition}
Problem (\ref{eq:disjoint_optimization}) has a feasible solution.
\end{proposition}
\begin{proof}
We construct a feasible solution. Let $u_1, \dots , u_n$ be an ordering of objects from $U$. We apply the following procedure.
\begin{itemize}
    \item[1)]  $\beta_{u_1}$ is a random value from $[0,1]$.
    \item[2)] For $1 < i \leq n$, $\beta_{u_i} = \min \{ I(\beta_{u_j}, M(u_j, u_i)) ; j < i \}$.
\end{itemize}
The adjacency property is obvious from the construction. We have to prove the granularity property. Let $u_i$ and $u_k$ be two objects for which $k < i$. Since $\beta_{u_i} = \min_{j < i} I(\beta_{u_j}, M(u_i, u_j))$, it holds that $\beta_{u_i} \leq I(\beta_{u_k}, M(u_i, u_k))$. From the residuation property, this is equivalent to $T(\beta_{u_i}, \beta_{u_k}) \leq M(u_i, u_k)$.
\end{proof}

For different IMTL fuzzy connectives and for different loss functions $L$, problem (\ref{eq:disjoint_optimization}) may take forms that cannot be efficiently solved in practice. However, we will consider the problem for $L$ being a scaled form of $MAE$ or $MSE$ (loss functions (\ref{eq:mae}) and (\ref{eq:mse})), and $T$ being isomorphic to the Łukasiewicz $t$-norm.

For such fuzzy connectives, the constraints of (\ref{eq:disjoint_binary_optimization}) are expressed as
\begin{align*}
    &\varphi^{-1}(\max(\varphi(\beta_u) + \varphi(\beta_v) - 1, 0)) \leq M(u,v)\\
    \Leftrightarrow & \max(\varphi(\beta_u) + \varphi(\beta_v) - 1, 0) \leq  \varphi(M(u,v))\\
    \Leftrightarrow & \varphi(\beta_u) + \varphi(\beta_v) \leq 1 + \varphi(M(u,v)),
\end{align*}
for isomorphism $\varphi$ and for $u,v \in U$. We introduce new variables $\forall u \in U, \, \alpha_u = \varphi(\beta_u)$ and $\forall u,v \in U, \, M_{\varphi}(u,v) = \varphi(M(u,v))$. With the new notations, the previous constraints may be expressed as
\begin{align*}
  \alpha_u + \alpha_v \leq 1 + M_{\varphi}(u,v).
\end{align*}
For the scaled mean absolute error $L_{MAE, \varphi}$, the objective function becomes
$$
\sum_{u \in U} |\varphi(1) - \varphi(\beta_u)| = |U| - \sum_{u \in U} \alpha_u,
$$
which leads to the optimization problem

\begin{equation}
\label{eq:T_L_linear}
\begin{array}{ll@{}ll}
&\text{maximize}  && \displaystyle\sum_{u \in U}  \alpha_u \\
&\text{subject to}    &&\alpha_u + \alpha_v \leq 1 + M_{\varphi}(u,v), \quad  u, v \in U\\
  &              && 0 \leq \alpha_u \leq 1, \quad u \in U.
\end{array}
\end{equation}
Optimization problem (\ref{eq:T_L_linear}) can be solved efficiently using linear programming techniques like the simplex method \cite{wolfe1959simplex}. 

For the scaled mean squared error $L_{MSE, \varphi}$, the objective function becomes
$$
\sum_{u \in U} (\varphi(1) - \varphi(\beta_u))^2 = \sum_{u \in U} (1 - \alpha_u)^2,
$$
which leads to the optimization problem

\begin{equation}
\label{eq:T_L_quadratic}
\begin{array}{ll@{}ll}
&\text{maximize}  && \displaystyle\sum_{u \in U} (1 - \alpha_u)^2\\
&\text{subject to}    &&\alpha_u + \alpha_v \leq 1 + M_{\varphi}(u,v), \quad  u, v \in U\\
  &              && 0 \leq \alpha_u \leq 1, \quad u \in U.
\end{array}
\end{equation}

Optimization problem (\ref{eq:T_L_quadratic}) can be solved efficiently using quadratic programming techniques like the simplex method variation for quadratic programming \cite{gass2003linear}.

\begin{figure}[!htb]
    \centering
    \includegraphics[width = 1\textwidth]{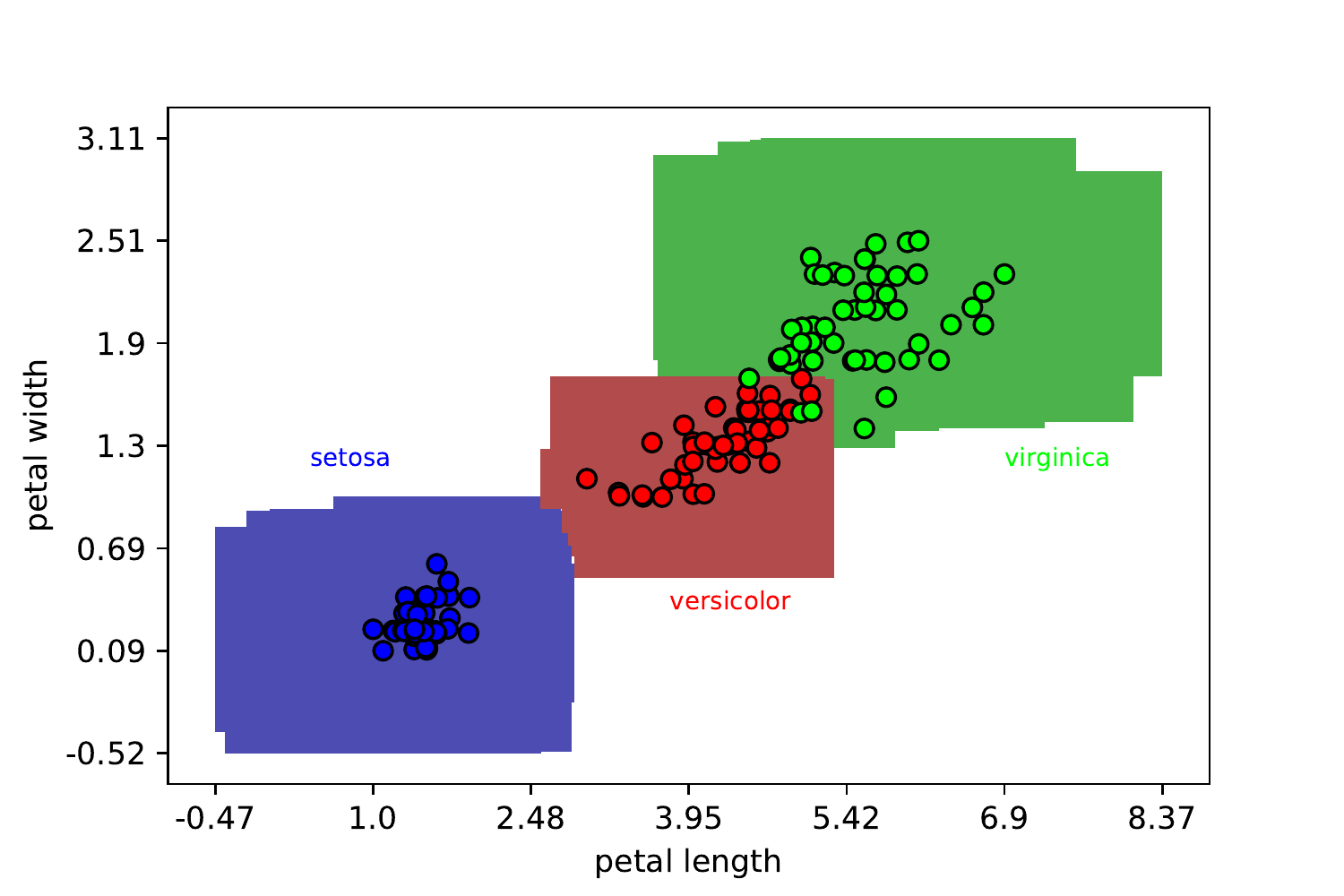}
    \caption{An example of the multi-class granular approximation on iris dataset constructed with relation (\ref{eq: triangular similarity})}
    \label{fig:iris_rectangular_granules}
\end{figure}
\begin{example}
We may see another example in Figure \ref{fig:iris_rectangular_granules}, where the objects come from the well-known iris dataset with, in this case, two features (petal length and petal width) and three classes (setosa, versicolor and virginica). The multi-class granular approximation is calculated by solving problem (\ref{eq:T_L_quadratic}) for $\widetilde{R}$ being triangular similarity (\ref{eq: triangular similarity}), and the granules are depicted using the obtained solution. On this figure, we can observe how granules look on the larger scale (in this case 150 objects). 
\end{example}

We provide another example with more complex shapes of granules.
\begin{example}
Consider a family of fuzzy relations defined as
\begin{equation}
\label{eq:general_similarity}
\widetilde{R}(u,v) = \max \bigg (1 - \frac{d(u,v)}{a}, 0 \bigg ),
\end{equation}
where $d$ is a metric (or distance function) on $U$ and $a$ is a positive real value (a parameter). It is easy to verify that such fuzzy relations are $T_L$-equivalences. More on the relationship between metrics and $T$-equality relations can be found in \cite{de2002metrics}. We now consider the Mahalanobis distance defined as \cite{mahalanobis1936distance}: 
\begin{equation}
    d(u,v)_{\mathbf{\Sigma}} = \sqrt{(\mathbf{u} - \mathbf{v})^T \mathbf{\Sigma} (\mathbf{u} - \mathbf{v})},
\end{equation}
where $\mathbf{u}$ is a numerical vector representing condition attributes of instance $u$ while $\mathbf{\Sigma}$ is a symmetric and positive-definite matrix. If $\mathbf{\Sigma}$ is an identity matrix, then $d(u,v)_{\mathbf{\Sigma}}$ is equal to the Euclidean distance. 

It is also easy to verify that the shape of level sets of granules, used to represent them in 2 dimensions, are in the case of family (\ref{eq:general_similarity}) equal to the shape of equidistant points from the origin w.r.t. metric $d$. In the case of the Mahalanobis distance, the shape of granules will be elliptical. The axis of such ellipses is controlled by the eigenvalues of $\mathbf{\Sigma}$ while the rotation is controlled by the eigenvectors of $\mathbf{\Sigma}$.
\begin{figure}[!htb]
    \centering
    \includegraphics[width = 1\textwidth]{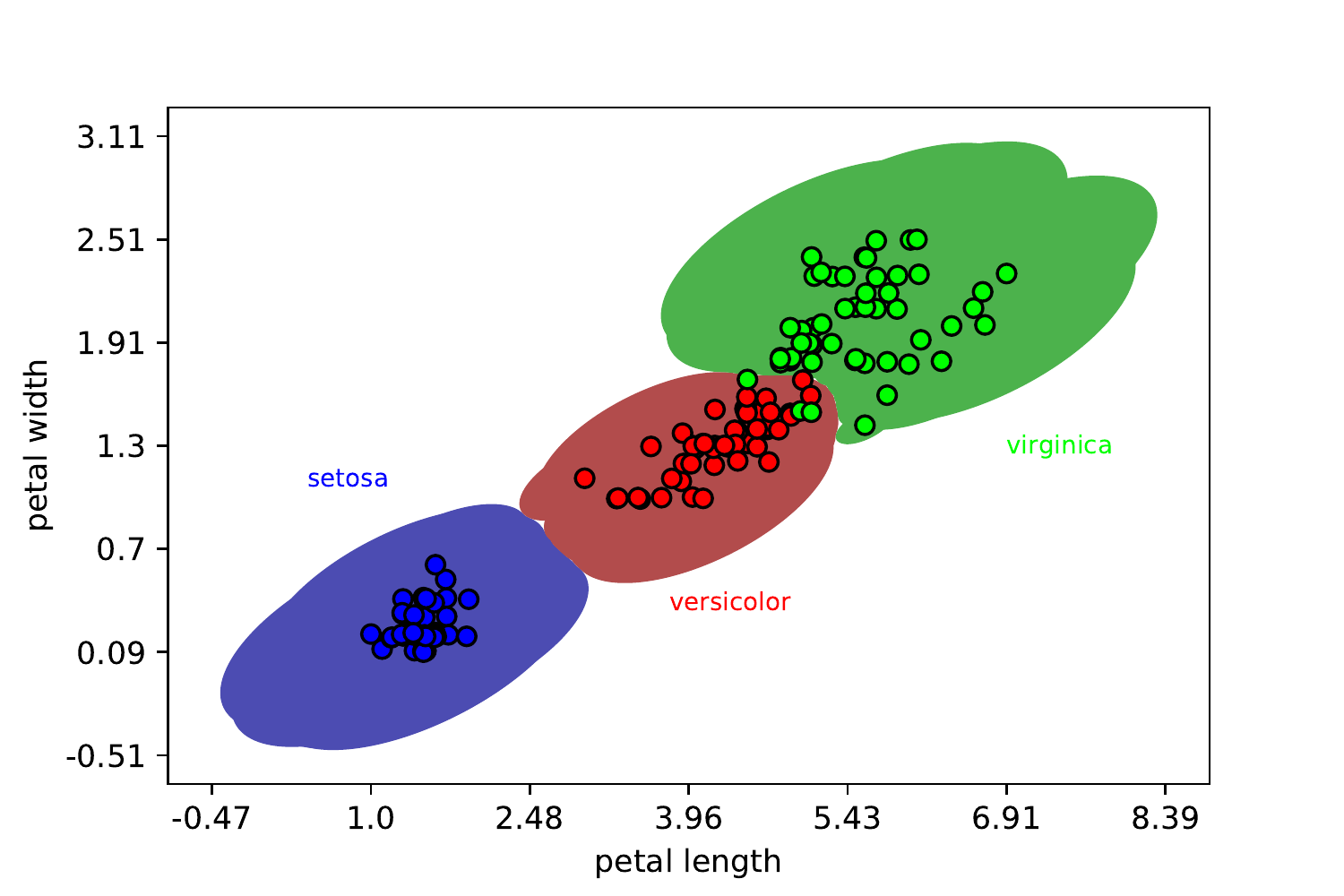}
    \caption{An example of the multi-class granular approximation on iris dataset constructed with relation (\ref{eq:general_similarity})}.
    \label{fig:iris_elliptical_granules}
\end{figure}

In Figure \ref{fig:iris_elliptical_granules}, we present an example of granules from the multi-class granular approximation calculated by solving (\ref{eq:T_L_quadratic}) and by using fuzzy relation (\ref{eq:general_similarity}) with $d$ being the Mahalanobis distance. The approximation is calculated on the iris dataset with two attributes and three decision classes as described above. The granules have the elliptical shape where the ratio of width and height of the ellipses is $2:1$. The rotation angle in this case is $45^{\circ}$. 
\end{example}

In Figures \ref{fig:iris_rectangular_granules} and \ref{fig:iris_elliptical_granules}, we can observe that some green points are depicted without their granules and are completely surrounded by the granules of red points. This basically means that the multi-class granular approximation values of these green points is smaller than 0.5 (hence, the granules cannot be drawn), and that the red granules are covering those green points. Therefore, it is suitable to change the labels of those green points into red. We can conclude that the learning, characterized by optimization problems (\ref{eq:T_L_linear}) and (\ref{eq:T_L_quadratic}), can be applied in classification problems and lead to an optimal relabeling of instances based on the loss function $L$.

\section{Conclusion and Future work}
\label{sec:conclusion_and_future_work}
We have introduced the concepts of disjoint and adjacent fuzzy granules and discussed their connection with the concept of granular approximation introduced before. Based on disjoint and adjacent granules, a granular approximation concept was applied to the multi-class classification problem leading to the definition of a multi-class granular approximation. At the end, we explained how to calculate it efficiently in practice for the Łukasiewicz $t$-norm and other fuzzy connectives that it generates, using linear and quadratic programming methods.


We consider the following future research directions.
\begin{itemize}
    \item We discussed in Section \ref{sec:calculation} that newly obtained multi-class granular approximations can be used for prediction purposes. The main focus of the future work will be to explore that more in depth i.e., to investigate the performance of optimization procedures (\ref{eq:T_L_linear}) and (\ref{eq:T_L_quadratic}) in classification problems. The goal will be to select classification datasets, set up a benchmark and compare the performance of the approaches (\ref{eq:T_L_linear}) and (\ref{eq:T_L_quadratic}) with other classification approaches. 
    \item In order to define a multi-class granular approximation, we used only symmetric loss functions compared to the original granular approximation that considers also non-symmetric loss functions (e.g., $p$-quantile loss) \cite{palangetic2021novel}. It would be worth investigating if application of a non-symmetric loss can help to control ``preference" toward some decision class (which is important in, e.g., imbalanced classification problems) and how such loss functions would perform in classification tasks. 
    \item In this article, $S$ is a crisp equivalence relation. However, there is an option to extend it to a fuzzy relation in order to handle regression tasks. In such a scenario, the understanding of granules and other properties is different than for $S$ being crisp. Hence, more analysis is required in the case of $S$ being a fuzzy relation.
    
    \item It would be worth to explore if the extension as for (\ref{eq:disjoint_optimization}) can be obtained for a $T$-preorder relation instead of a $T$-equivalence relation.
    Such extension could be then applied in monotone classification problems.
\end{itemize}

\section*{Acknowledgements}

Marko Palangetić and Chris Cornelis would like to thank Odysseus project from Flanders Research Foundation (FWO), grant no. G0H9118N, for funding their research.
Salvatore Greco wishes to acknowledge the support of the Ministero dell’Istruzione, dell’Universitá e della Ricerca (MIUR) - PRIN 1576 2017, project “Multiple Criteria Decision Analysis and Multiple Criteria Decision Theory”, grant 2017CY2NCA. 
Roman Słowiński is acknowledging the support of grant 0311/SBAD/0700.



\bibliographystyle{elsarticle-num} 


\bibliography{bibliography}
\end{document}